\documentclass{article}

\usepackage{microtype}
\usepackage{graphicx}
\usepackage{subfigure}
\usepackage{booktabs} 

\usepackage{hyperref}

\usepackage[accepted]{icml2018_arxiv}

\usepackage{amsmath}
\usepackage{amssymb}
\usepackage{amsthm}
\usepackage{bm}
\usepackage{graphicx} 
\usepackage{subfigure} 
\usepackage{multirow}
\usepackage{color}
\usepackage{natbib}
\usepackage{algorithm}
\usepackage{algorithmic}
\usepackage{multirow}

\newtheorem{theorem}{Theorem}[section]

\newtheorem{corollary}[theorem]{Corollary}
\newtheorem{definition}[theorem]{Definition}

\icmltitlerunning{Superposition-Assisted Stochastic Optimization for Hawkes Processes}

\begin{document}

\twocolumn[
\icmltitle{Superposition-Assisted Stochastic Optimization for Hawkes Processes}



\icmlsetsymbol{equal}{*}

\begin{icmlauthorlist}
\icmlauthor{Hongteng Xu}{inf,du}
\icmlauthor{Xu Chen}{ts}
\icmlauthor{Lawrence Carin}{du}

\end{icmlauthorlist}

\icmlaffiliation{du}{Department of Electrical and Computer Engineering, Duke University, Durham, NC, USA}
\icmlaffiliation{inf}{InfiniaML, Inc., Durham, NC, USA}
\icmlaffiliation{ts}{School of Software, Tsinghua University, Beijing, China}

\icmlcorrespondingauthor{Hongteng Xu}{hongteng.xu@infiniaml.com}

\icmlkeywords{Superposition, Stochastic optimization, Hawkes process}

\vskip 0.3in
]



\printAffiliationsAndNotice{}  

\begin{abstract}
We consider the learning of multi-agent Hawkes processes, a model containing multiple Hawkes processes with shared endogenous impact functions and different exogenous intensities. 
In the framework of stochastic maximum likelihood estimation, we explore the associated risk bound. 
Further, we consider the superposition of Hawkes processes within the model, and demonstrate that under certain conditions such an operation is beneficial for tightening the risk bound. 
Accordingly, we propose a stochastic optimization algorithm assisted with a diversity-driven superposition strategy, achieving better learning results with improved convergence properties. 
The effectiveness of the proposed method is verified on synthetic data, and its potential to solve the cold-start problem of sequential recommendation systems is demonstrated on real-world data. 
\end{abstract}

\section{Introduction}
\label{sec:intro}
Real-world sequential data often describe interactions between a set of independent agents and the entities in a system. 
For each agent, its event sequence records when it interacts with which system entities. 
A typical example is online shopping, in which the agents correspond to the users, the entities correspond to the items, and the event sequences collect when and which items the users purchase. 
From the viewpoint of point processes, these sequences can be modeled by multiple Hawkes processes~\citep{hawkes1971point}, which simultaneously capture the endogenous triggering patterns among the entities and the exogenous factors related to agents. 
When triggering patterns are shared across different agents, the collection of these Hawkes processes formulates a new model called \emph{multi-agent Hawkes process}, investigated in this paper. 

Learning multi-agent Hawkes processes suffers from the challenges of scalability and robustness. 
The target systems may involve a large number of entities and make interactions with numerous agents, $i.e.$, the online shopping system mentioned above.
However, most existing learning methods~\citep{lewis2011nonparametric,bacry2014second,achab2016uncovering,xu2016learning,yang2017online} are time-consuming, with computational complexity that is at least the square of the number of observed events. 
The model may be unreliable when the number of events is limited, because the dimension of model parameters (the complexity of the model) increases linearly with respect to the number of agents and quadratically respect to the number of entities. 

To address these problems, we propose a new learning method for multi-agent Hawkes process, endowed with good scalability and robustness. 
In particular, a stochastic maximum likelihood estimation is applied, with improvements in its stability and speed. Further, we derive the risk bound corresponding to this method. 
We also consider the superposition of Hawkes processes within our model, and prove that this superposition can help tighten the risk bound. 
Accordingly, a superposition-assisted stochastic optimization algorithm is designed to learn the proposed model robustly. 

To the authors' knowledge, this work is the first stochastic optimization method for Hawkes process-related models that has an explicit analysis on the risk bound and the superposition operation. 
Moreover, distinct from traditional work~\citep{cox1954superposition,ccinlar1968superposition,moller2012transforming}, our method treats the superposition of Hawkes processes as a ``benefit'' rather than a ``challenge'' in the learning task. 
Experiments show that the proposed method achieves lower estimation errors with improved convergence, and obtains encouraging results for solving the cold-start problem of recommendation systems.

\section{Learning Multi-Agent Hawkes Processes}
\subsection{Proposed model}
Consider a system with $C$ entities and the activities of $M$ independent agents. 
For agent $m\in\{1,...,M\}$, we observe the associated behavior as an $I_m$-element event sequence, $\bm{s}^m=\{(t_i^m, c_i^m)\}_{i=1}^{I_m}$. 
Here, $t_i^m\in [0, T]$ is the timestamp of the event, and $c_i^m\in\{1,...,C\}$. 
For agent $m$, its counting process is $N_c^m(t)=|\{(t_i^m, c_i^m)|t_{i}^m\leq t,c_i^m=c~i=1,2,...\}|$ for $c\in\{1,\dots,C\}$, where $|\cdot|$ calculates the cardinality of a set. 
Each $N_c^m(t)$ is the number of agent-$m$ events involving entity $c$ up to time $t$. 
Furthermore, denote the historical events for agent $m$ up to time $t$ as $\mathcal{H}_t^m$. 
The expected instantaneous happening rate of an event with entity $c$ is represented by an intensity function:
\begin{eqnarray}\label{intensity}
\begin{aligned}
\lambda_c^m(t)={\mathbb{E}[dN_c^m(t)|\mathcal{H}_t^m]}/{dt},
\end{aligned}
\end{eqnarray}

The multi-agent Hawkes process describes the target system as an endogenously-stationary system with exogenous fluctuations among agents. 
In particular, we assume that the happening rate of events is controlled by exogenous factors dependent on agent attributes and endogenous factors depending on the triggering patterns among different entities. 
Accordingly, we formulate the intensity function as
\begin{eqnarray}\label{mshp}
\begin{aligned}
\lambda_c^m(t)=\mu_c^m + \sideset{}{_{(t_i^m, c_i^m)\in\mathcal{H}_t^m}}\sum\phi_{cc_i}(t,t_i).
\end{aligned}
\end{eqnarray}
Here, $\mu_{c}^m$ is the exogenous intensity of entity $c$ related to agent $m$. 
The term $\sum_{(t_i^m, c_i^m)\in\mathcal{H}_t^m}\phi_{cc_i}(t,t_i)$ represents the accumulation of endogenous intensity caused by history~\citep{farajtabar2014shaping}. 
The impact function $\phi_{cc'}(t,s)$, $t\geq s$, represents the influence of entity $c'$ on entity $c$ when their corresponding events happen at time $s$ and $t$, respectively, which is stationary and invariant to different agents. 

{As in most existing work~\citep{zhou2013learning,zhao2015seismic,bacry2015generalization,xu2016learning}}, we assume the impact function to be shift-invariant, $i.e.$, $\phi_{cc'}(t,s)=\phi_{cc'}(t-s)$, and parametrize it by a basis representation, $i.e.$, $\phi_{cc'}(t)=\sum_{l=1}^{L}a_{cc'l}g_{l}(t)$, where $\{g_{l}(t)\}_{l=1}^{L}$ are predefined nonnegative kernels, like Gaussian or exponential functions. 
To obtain a physically-meaningful intensity function, we require all $\mu_c^m$ and $a_{cc'l}$ to be nonnegative. 

We denote the multi-agent Hawkes process with $M$ independent agents as $HP(\bm{U},\bm{A})$, with parameters defined by exogenous intensity matrix $\bm{U}=[\mu_{c}^m]\in\mathbb{R}^{C\times M}$ and an endogenous impact tensor $\bm{A}=[a_{cc'l}]\in\mathbb{R}^{C\times C\times L}$. 
For agent $m$, its instantiated counting process is $N^m(t)=[N_c^m(t)]\sim HP(\bm{\mu}^m,\bm{A})$, where $\bm{\mu}^m=[\mu_1^m;...;\mu_C^m]$. 

Note that the intensity function in (\ref{mshp}) corresponds to a linear function of $\{\bm{U},\bm{A}\}$:
\begin{eqnarray}\label{linear}
\begin{aligned}
\lambda_c^m(t) = \bm{x}_{c,m}^{\top}(t)\bm{\theta},~\mbox{and}~\bm{\theta}=[\mbox{vec}(\bm{U});\mbox{vec}(\bm{A})],
\end{aligned}
\end{eqnarray}
where $\bm{\theta}\in\mathbb{R}^{C(M+CL)}$ is a vectorized representation of the parameters, with $\mbox{vec}(\cdot)$ vectorizing the input.
$\bm{x}_{c,m}(t)=[\mbox{vec}(\bm{e}_{c,m});\mbox{vec}(\bm{E}_{c,m}(t))]$. 
$\bm{e}_{c,m}\in\mathbb{R}^{C\times M}$, whose elements are zeros except that in the $c$-th row and the $m$-th column, which has value $1$, corresponds to $\mu_c^m$. 
$\bm{E}_{c,m}(t)=[e_{cc'l}^m(t)]\in\mathbb{R}^{C\times C\times L}$, where $e_{cc'l}^m(t)=\sum_{(t_i^m, c_i^m)\in\mathcal{H}_t^m,~c_i^m=c'}g_l(t-t_i^m)$. 
Additionally, we further parametrize $\bm{A}$ when additional features of entities $\{\bm{f}_c\}_{c\in\mathcal{C}}$ are available, $e.g.$, $a_{cc'l}=\bm{f}_c^{\top}\bm{W}_{l}\bm{f}_{c'}$ and $a_{cc'l}=\bm{w}_{l,c}^{\top}\bm{f}_{c'}$~\citep{du2015time,wang2016coevolutionary}.
In those cases, we can still write $\lambda_c^m(t)$ linearly like (\ref{linear}) with $\bm{\theta}=\mbox{vec}(\bm{U},\{\bm{W}_l\})$ or $\mbox{vec}(\bm{U},\{\bm{w}_{l,c}\})$, and our theoretical analysis below is also applicable. 
Therefore, without loss of generality, in this work we only consider the model without additional features. 

This multi-agent Hawkes process model may be applied to many real-world phenomena.  
The shopping behavior of users are mainly dependent on their intrinsic preferences on different items, as well as the triggering patterns among historical items~\citep{wang2016coevolutionary,xu2017benefits}. 
{Besides online shopping, the information diffusion process in social networks obeys the same mechanism~\citep{farajtabar2014shaping,gomez2015estimating}.}
The event sequences corresponding to different information ($i.e.$, agents) diffusing among different users ($i.e.$, entities) share the same endogenous triggering patterns among users because the relationships between users are stationary.
However, they may have different exogenous intensities, depending on diverse user preferences relative to different information.

\subsection{Stochastic maximum likelihood estimation}
Without loss of generality, in the following we assume that each agent generates one sequence. 
Given $M$ event sequences $\{\bm{s}^m\}_{m=1}^{M}$, we learn an $M$-agent Hawkes process by maximum likelihood estimation:
\begin{eqnarray}\label{opt}
\begin{aligned}
\sideset{}{_{\bm{\theta}\geq \bm{0}}}\min \mathcal{L}(\bm{\theta}) + \alpha\mathcal{R}(\bm{\theta}),
\end{aligned}
\end{eqnarray}
where $\mathcal{L}(\bm{\theta})$ represents the negative log-likelihood. 
According to the definition in~\citep{daley2007introduction} and the linear representation of intensity function in (\ref{linear}), we write $\mathcal{L}(\bm{\theta})$ as a sum of random functions:
\begin{eqnarray}\label{nll}
\begin{aligned}
&\mathcal{L}(\bm{\theta})=\sum_{m=1}^{M}\sum_{i=1}^{I_m}\Biggl\{\sum_{c=1}^{C}\int_{t_{i-1}^m}^{t_i^m}\lambda_c^m(s)ds-\log\lambda_{c_i^m}(t_i^m)\Biggr\}\\
&=\sum_{m,i}\Bigl\{\bm{X}_m^{\top}(t_i^m)\bm{\theta}-\log\bm{x}_{c_i^m,m}^{\top}(t_i^m)\bm{\theta}\Bigr\}=\sum_{m,i}f_i^m(\bm{\theta}),
\end{aligned}
\end{eqnarray}
where $\bm{x}_{c_i^m,m}(t_i^m)$ is defined as that in (\ref{linear}), and $\bm{X}_m(t_i^m)=\sum_{c=1}^{C}\int_{t_{i-1}^m}^{t_i^m}\bm{x}_{c,m}(s)ds$. 
For each event, the negative logarithm of its conditional probability given history is denoted $f_i^m(\bm{\theta})=-\log p((t_i^m, c_i^m)|\mathcal{H}^m_{t_i^m})$, which is nonnegative and strongly-convex. 
$\mathcal{R}(\bm{\theta})$ is an optional convex regularizer on the parameters, $e.g.$, the sparse, group-sparse and low-rank regularizers imposed on $\bm{A}$~\citep{xu2016learning}. 
These regularizers can often be reformulated as square loss terms by introducing auxiliary and dual variables~\citep{zhou2013learning,luo2015multi}, which do not change the convexity of the problem.
Therefore, for convenience we only consider the likelihood term without regularizers in the following theoretical discussion.

The optimization problem in (\ref{opt}) can be solved by several iterative algorithms, $e.g.$, the alternating direction method of multipliers (ADMM)~\citep{zhou2013learning}, the projected gradient descent~\citep{lewis2011nonparametric}, etc. 
However, most existing methods merely consider the batch optimization of (\ref{opt}), which use all observed events to calculate gradients in each iteration. 
Because the intensity at each timestamp considers all historical events, the computational complexity per iteration is as high as $\mathcal{O}(MI^2)$, where $I=\arg\max_m\{I_m\}_{m=1}^M$, which leads to a scalability challenge. 

To accelerate the learning process, we design an event-level stochastic optimization algorithm. 
In each iteration, we select a small batch of events randomly from the sequences to calculate the gradients, and then update the parameters by {projected gradient descent~\citep{lewis2011nonparametric,du2015time}. }
The setup for our event-level stochastic optimization algorithm is summarized in Algorithm~\ref{alg1}, where the operator $(\cdot)_+$ sets all negative inputs to zeros.

\begin{algorithm}[t]
   \caption{Stochastic Optimization of Multi-agent Hawkes Processes (StocOpt)}
   \label{alg1}
\begin{algorithmic}[1]
   \STATE \textbf{Input:} Event sequences $\{\bm{s}^m\}_{m=1}^{M}$, batch size $B$, length of history $J$, offset $\lambda_0$, and learning rate $\eta$.
   \STATE \textbf{Output:} Parameters of model, $\bm{\theta}=[\mbox{vec}(\bm{U}), \mbox{vec}(\bm{A})]$.
   \STATE Initialize $\bm{U}=[\mu_c^m]$ and $\bm{A}=[a_{cc'l}]$ randomly.
   \REPEAT
   \STATE Select a batch of events $\mathcal{I}_{sub}^m$ randomly from $\{\bm{s}^m\}$.
   \STATE Calculate $\bm{x}_{c_i^m,m}(t_i^m)$, $\bm{X}_m(t_i^m)$ with at most $J$ historical events for $i\in \mathcal{I}_{sub}^m$, and obtain $f_i^m(\bm{\theta})$.
   \STATE $\bm{\theta}=(\bm{\theta}-\eta\sum_{i\in \mathcal{I}_{sub}^m}\nabla_{\lambda_0} f_i^m)_+$.
   \UNTIL{convergence}
\end{algorithmic}
\end{algorithm}

There are two aspects of Algorithm~\ref{alg1} that reduce its computational complexity and improve stability. 
First, in practice the impact functions often decay over time. 
This implies that one can ignore historical events happening far before the current ones when calculating $\bm{x}_{c_i^m,m}(t_i^m)$ and $\bm{X}_m(t_i^m)$ (line 6 of Algorithm~\ref{alg1}). 
Therefore, we set a maximum $J$ for the number of historical events used in $\bm{x}_{c_i^m,m}(t_i^m)$ and $\bm{X}_m(t_i^m)$. 
As a result, the computational complexity per iteration reduces to $\mathcal{O}(BJ)$, where both the batch size $B$ and the length of history $J$ are much smaller than $I$. 
Secondly, the original gradient with respect to each event is $\nabla f_i^m(\bm{\theta})=\bm{X}_m(t_i^m)-\frac{\bm{x}_{c_i^m,m}(t_i^m)}{\bm{x}_{c_i^m,m}^{\top}(t_i^m)\bm{\theta}}$, which may suffer from numerical problems when $\bm{x}_{c_i^m,m}^{\top}(t_i^m)\bm{\theta}\approx 0$.
To improve the stability of our algorithm, we introduce a small positive offset $\lambda_0$ and apply the gradient with $\lambda_0$, $i.e.$,
\begin{eqnarray}\label{grad}
\begin{aligned}
\nabla_{\lambda_0} f_i^m(\bm{\theta})=\bm{X}_m(t_i^m)-\frac{\bm{x}_{c_i^m,m}(t_i^m)}{\max\{\bm{x}_{c_i^m,m}^{\top}(t_i^m)\bm{\theta},\lambda_0\}}.
\end{aligned}
\end{eqnarray}
This means that our algorithm tends to learn each intensity function $\lambda_c^m(t)$ with a lower bound $\lambda_0$. 

\subsection{Risk bounds}
The functions $\{f_i^m\}$ in~(\ref{nll}) can be viewed as samples associated with an unknown distribution $\mathbb{P}$. 
If we normalize the negative log-likelihood in~(\ref{nll}) with the number of events, $i.e.,$, $\frac{1}{I_{\Sigma}}\mathcal{L}(\bm{\theta})$ with $I_{\Sigma}=\sum_m I_m$, then we use an empirical objective function $\frac{1}{I_{\Sigma}}\sum_{m,i}f_{i}^m$ to approximate an expected function $\mathbb{E}_{f\sim\mathbb{P}}[f(\bm{\theta})]$. 
From this viewpoint, our algorithm achieves stochastic convex optimization for a typical empirical risk minimization (ERM) problem. 
Denote the ground truth and the estimation of the proposed model as $\bm{\theta}_*$ and $\hat{\bm{\theta}}$, respectively. 
We can estimate the excess risk bound of our learning method, denoted as $\mathbb{E}[f(\hat{\bm{\theta}})-f(\bm{\theta}_*)]$, based on the properties of the random function $f$ and some simple and reasonable assumptions. 
\begin{theorem}\label{thm1}
Suppose that the multi-agent Hawkes process $HP(\bm{U},\bm{A})$ with $C$ entities and $M$ agents has intensity functions defined as (\ref{mshp}) and satisfies: 

\vspace{-3pt}
1. Its parameters $\bm{U}$ and $\bm{A}$ are nonnegative and bounded, $i.e.$, $\bm{U},\bm{A}\geq \bm{0}$, $\|\bm{U}\|_F^2\leq U_{0}$ and $\|\bm{A}\|_F^2\leq A_0$. 

\vspace{-3pt}
2. For all $t\in [0, T]$, $c\in\mathcal{C}$ and $m\in\mathcal{M}$, $\lambda_c^m(t)\geq\lambda_0$ . 

\vspace{-3pt}
3. The decay kernels have limited energy, $i.e.$, $g_l^2(t)\leq G$ for all $t\in[0,T]$ and $l=1,...,L$.

\vspace{-3pt}
Applying stochastic optimization to the ERM problem (\ref{opt}), with probability at least $1-2\delta$, $0<\delta<0.5$, we have
\begin{eqnarray}\label{bound1}
\begin{aligned}
\mathbb{E}[f(\hat{\bm{\theta}})-f(\bm{\theta}_*)]=\mathcal{O}\left( \frac{R\tau (D\log I_{\Sigma}+\log\frac{2}{\delta})}{I_{\Sigma}} \right),
\end{aligned}
\end{eqnarray}
where $I_\Sigma=\sum_m I_m$ is the number of observed events, $\tau=(\frac{L G I}{\lambda_0})^2$ and $I=\max_m \{I_m\}$ is the maximum number of events per sequence. $R=U_0+A_0$ and $D=C(M+CL)$ is the dimension of parameters.
\end{theorem}
\emph{Proof sketch.} According to the conditions above and the definition of our model, we can derive that 1) the parameters are bounded and in a convex domain; 2) the random functions $f$ are nonnegative and $\tau$-smooth over the domain; 3) the random functions $f$ and their expectation $\mathbb{E}[f]$ are $\tau$-strongly convex~\citep{bertsekas2015convex}. 
These properties ensure that we can take advantage of Theorem 5 in~\citep{zhang2017empirical} and specify it to the learning problem of multi-agent Hawkes process.
The detailed proof of Theorem~\ref{thm1} is given in the Supplementary Material.

\section{Assistance from Superposition}
\subsection{The effects of superposition on risk bounds}
Given Theorem~\ref{thm1}, the key problem is whether and how we can tighten the risk bound in (\ref{bound1}). 
As mentioned, we will show that under certain conditions superposing the Hawkes processes within the model can tighten the risk bound, and accordingly, improve learning results. 

In general, given a set of temporal point processes $\{N^m\}_{m=1}^{M}$, their superposition is a new point process $N$ defined by the sum of counting processes, $i.e.$, $N(t) = \sum_{m=1}^{M}N^m(t)$, $t\geq 0$. 
Accordingly, the instantiated event sequence of the superposed point process $N$ is the superposition of the sequences corresponding to $\{N^m\}_{m=1}^{M}$. 
For Hawkes processes, we have:
\begin{theorem}\label{thm2}
For $M$ independent Hawkes processes with shared impact functions, $i.e.$, $HP(\bm{\mu}^m,\bm{A})$ and $N^m(t)\sim HP(\bm{\mu}^m,\bm{A})$ for $m=1,...,M$, their superposition becomes a single Hawkes process, $i.e.$, $N(t)=\sum_{m=1}^{M}N^m(t)$ and $N(t)\sim HP(\sum_{m=1}^{M}\bm{\mu}^m,\bm{A})$. 
\end{theorem}
The proof of Theorem~\ref{thm2} is given in the Supplementary Material. 
Based on Theorem~\ref{thm2}, we have:
\begin{corollary}\label{coro1}
Given an $M$-agent Hawkes process, $i.e.$, $HP(\bm{U},\bm{A})$, we can generate an $M'$-agent Hawkes process $HP(\bm{U}',\bm{A})$ by randomly superposing different agents' processes. 
Here, $\bm{U}'=\bm{U}\bm{P}$ and $\bm{P}\in\{0,1\}^{M\times M'}$ is a binary matrix. 
For $m'=1,...,M'$, the ones in the $m'$-th row of $\bm{P}$ indicate the source processes of the $m'$-th new process.
\end{corollary}

Further, we can define a special kind of superposition:
\begin{definition}\label{def1}
For the $M$-agent Hawkes process $HP(\bm{U},\bm{A})$, its superposition $HP(\bm{U}', \bm{A})$ is called $K$-nonaugmented superposition if the binary matrix $\bm{P}$ in Corollary~\ref{coro1} satisfies $\|\bm{1}_{M}^{\top}\bm{P}\|_{\infty}=K$ and $\bm{P}\bm{1}_{M'}=\bm{1}_{M}$.
\end{definition}
Here $\|\cdot\|_\infty$ returns the maximum of input and $\bm{1}_M$ is an $M$-dimensional all-one vector. 
In practice, the $K$-nonaugmented superposition can be obtained by segmenting the $M$ sequences into $M'$ folders with maximum size $K$ and superposing the sequences in each folder. 
Such a superposition is ``nonaugmented'' because it does not reuse any observed events. 

Theorem~\ref{thm1} and Corollary~\ref{coro1} provide insights into the relationship between the superposition and the risk bound. 
In particular, the following theorem points out the condition for tightening the risk bound with the help of superposition.
\begin{theorem}\label{thm3}
Suppose that we have an $M$-agent Hawkes process $HP(\bm{U},\bm{A})$ with $C$ entities, satisfying the constraints in Theorem~\ref{thm1}.
Given its $K$-nonaugmented superposition, $i.e.$, an $M'$-agent Hawkes process $HP(\bm{U}',\bm{A})$ with new parameters $\bm{\theta}'=[\mbox{vec}(\bm{U}');\mbox{vec}(\bm{A})]\in\mathbb{R}^{D'}$ and $D'=C(M'+CL)$, with probability at least $1-2\delta$, $0<\delta<0.5$, we can learn the new model with a tighter risk bound, $i.e.$, $\mathbb{E}[f'(\hat{\bm{\theta}}')-f'(\bm{\theta}^{'}_*)]\leq \mathbb{E}[f(\hat{\bm{\theta}})-f(\bm{\theta}_*)]$, if the upper bound of $\|\bm{U}'\|_F^2$, denoted as $U^{'}_0$, satisfies
\begin{eqnarray}\label{cond}
\begin{aligned}
U^{'}_0\leq (A_0+U_0)\frac{(M+CL)\log I_\Sigma + \log\frac{2}{\delta}}{(M'+CL)\log I_\Sigma +\log\frac{2}{\delta}} - A_0.
\end{aligned}
\end{eqnarray}
Here $f'$ is the random function of the new model, and $\{A_0, U_0, L, I_\Sigma\}$ are defined as in Theorem~\ref{thm1}.
\end{theorem} 
\begin{proof}
Without loss of generality, we assume that $K$ event sequences are superposed in each folder, and accordingly, $K=\frac{M}{M'}$. 
Because the superposed process is still a multi-agent Hawkes process, the risk bound in Theorem~\ref{thm1} is still applicable, and we have 
\begin{eqnarray}\label{bound2}
\begin{aligned}
\mathbb{E}[f'(\hat{\bm{\theta}}')-f'(\bm{\theta}^{'}_*)]=\mathcal{O}\left( \frac{R'\tau' (D'\log I_{\Sigma}+\log\frac{2}{\delta})}{I_{\Sigma}} \right).
\end{aligned}
\end{eqnarray}
Here, $0<\delta<0.5$, $I_\Sigma$ is unchanged because superposition does not change the total number of events. 
$R' = U^{'}_0 + A_0$ is the upper bound of $\|\bm{\theta}'\|_2^2$.
$D'=C(M'+CL)$ is the dimension of $\bm{\theta}'$. 
$\tau'=({L G I'}/{\lambda^{'}_0})^2$, where $I'$ is the maximum number of events per sequence and the $\lambda_0^{'}$ is the lower bound of the intensity function for the superposed process. 
The number and the upper bound of decay kernels, $i.e.$, $L$ and $G$, are unchanged. 
Because each new sequence contains $K$ original sequences, we have $I'=KI$ and $\lambda^{'}_0=K\lambda_0$. 
As a result, $\tau'$ is equal to the $\tau$ in~(\ref{bound1}). 
According to (\ref{bound1},\ref{bound2}), if $\mathbb{E}[f'(\hat{\bm{\theta}}')-f'(\bm{\theta}^{'}_*)]\leq \mathbb{E}[f(\hat{\bm{\theta}})-f(\bm{\theta}_*)]$, we have
\begin{eqnarray*}
\begin{aligned}
R'(D'\log I_\Sigma + \log\frac{2}{\delta})\leq R(D\log I_\Sigma + \log\frac{2}{\delta}).
\end{aligned}
\end{eqnarray*}
This completes the proof.
\end{proof}

\subsection{Superposition-assisted stochastic optimization}\label{ssec:diverse}
According to Theorem~\ref{thm3}, by superposing the sequences generated by different agents, we can reduce the number of parameters from $D$ to $D'$, simplifying the learning task. 
However, the upper bound of parameters increases commensurately, enlarging the search space, having a negative influence on minimizing empirical risk. 
Therefore, to tighten risk bounds by superposition, we need to ensure that the increase of the parameters caused by superposition will not counteract the benefits from dimensionality reduction. 

For the multi-agent Hawkes process with orthogonal exogenous intensities, we can tighten the risk bound above with confidence because of the following corollary:
\begin{corollary}\label{coro2}
For the $M$-agent Hawkes process $HP(\bm{U},\bm{A})$, if its exogenous intensities $\bm{\mu}^m$'s in $\bm{U}$ are orthogonal, $i.e.$, $(\bm{\mu}^m)^{\top}\bm{\mu}^{m'}=0$ for all $m\neq m'$, learning from its $K$-nonaugmented superposition always has a tighter risk bound. 
\end{corollary}
\begin{proof}
For $HP(\bm{U},\bm{A})$, its $K$-nonaugmented superposition $HP(\bm{U}',\bm{A})$ satisfies $\|\bm{U}'\|_F^2=\|\bm{U}\|_F^2$ because of the orthogonal property. 
Because $U^{'}_0=U_0$ and $M'<M$, the inequality (\ref{cond}) always holds.
\end{proof}

\begin{algorithm}[t]
   \caption{Diversity-driven superposition}
   \label{alg2}
\begin{algorithmic}[1]
   \STATE \textbf{Input:} Event sequences $\mathcal{S}=\{\bm{s}^m\}_{m=1}^{M}$.
   \STATE \textbf{Output:} Superposed sequences $\mathcal{S}'=\{\bm{s}^{m'}\}_{m'=1}^{M'}$.
   \STATE Initialize $\mathcal{S}'=\emptyset$, $K=\frac{M}{M'}$.
   \STATE Initialize a distribution $\bm{p}=[p_1;...;p_M]$, $p_m=\frac{1}{M}$.
   \STATE For $m=1,...,M$, estimate the exogenous intensity $\hat{\bm{\mu}}^m$ as $[\frac{N_1^m(T)}{T};...,;\frac{N_C^m(T)}{T}]$, and $\widehat{\bm{U}}=[\hat{\bm{\mu}}^m]$.
   \STATE \textbf{for} {$m'=1:M'$} \textbf{do}
   \STATE \quad Sample a sequence $\bm{s}^m$ from $S$ with $m\sim \bm{p}$. 
   \STATE \quad Initialize $\bm{s}^{m'}=\bm{s}^m$, $\hat{\mu}^{m'}=\hat{\bm{\mu}^m}$, and set $p_m=0$. 
   \STATE \quad \textbf{for} {$k=1:K-1$} \textbf{do}
   \STATE \quad\quad Re-weight $\bm{p}=\bm{p}\odot\exp(-\widehat{\bm{U}}^{\top}\hat{\mu}^{m'})$; 
   \STATE \quad\quad Normalize $\bm{p}=\frac{1}{\|\bm{p}\|_1}\bm{p}$.
   \STATE \quad\quad Sample a sequence $\bm{s}^{m_k}$ from $S$ with $m_k\sim \bm{p}$.
   \STATE \quad\quad $\bm{s}^{m'}=\bm{s}^{m'}\cup \bm{s}^{m_k}$, $p_{m_k}=0$.
   \STATE \quad $\mathcal{S}'=\mathcal{S}'\cup \bm{s}^{m'}$.
\end{algorithmic}
\end{algorithm}

Although the orthogonality between arbitrary exogenous intensities is a very strong condition, fortunately, in real-world data it may be achieved approximately between the groups of exogenous intensities. 
Recall the online shopping case. 
The user preferences ($i.e.$, exogenous intensities) typically have clustering structure, and the preferences belonging to different groups are very diverse, with a focus on different items. 
Taking this fact into account, the key problem for real-world data is how to find the sequences with diverse exogenous intensities.
Based on Corollary~\ref{coro2}, we use the orthogonality between exogenous intensities as a kind of measurement of diversity, and propose a diversity-driven superposition strategy in Algorithm~\ref{alg2} to suppress the upper bound of parameters and tighten the risk bound accordingly.

In Algorithm~\ref{alg2}, the operator ``$\cup$'' between sequences (line 13) represents superposition and ``$\odot$'' (line 10) means elementwise multiplication. 
The estimation of exogenous intensity (line 5) is based on fact that $\mu_c^m\propto \frac{N_c^m(T)}{T}$ as $T\rightarrow \infty$~\citep{zhu2013ruin}. 
According to the estimation, we calculate the diversities between the target superposed sequence and the remaining source sequences and re-weight the sampling distribution (line 10). 
In particular, given the target superposed sequence $\bm{s}_{m'}$, the $m$-th element of $\widehat{\bm{U}}^{\top}\hat{\mu}^{m'}$ is the inner product between $\hat{\mu}^m$ and $\hat{\mu}^{m'}$. 
Similar to the method in~\citep{wachinger2015diverse}, when the diversity is low ($i.e.$, $(\hat{\bm{\mu}}^m)^{\top}\hat{\mu}^{m'}$ is large) its sampling probability $p_m$ multiplies by small factor $\exp(-(\hat{\bm{\mu}}^m)^{\top}\hat{\mu}^{m'})$ and thus decreases significantly. 
On the contrary, those with high diversities are more likely to be superposed into the target sequence. 
Because the sequences superposed together have small inner products between their exogenous intensities, the new parameters's upper bound increases slowly and the inequality (\ref{cond}) is more likely to hold.

\begin{algorithm}[t]
   \caption{Superposition-assisted stochastic optimization}
   \label{alg3}
\begin{algorithmic}[1]
   \STATE \textbf{Input:} Event sequences $\mathcal{S}=\{\bm{s}^m\}_{m=1}^{M}$.
   \STATE \textbf{Output:} Parameters $\bm{A}$ and $\bm{U}$.
   \STATE Obtain superposed sequences $\mathcal{S}'$ by Algorithm~\ref{alg2}.
   \STATE \textbf{repeat}
   \STATE Learn $HP(\bm{U}',\bm{A})$ from $\mathcal{S}'$ by Algorithm~\ref{alg1}.
   \STATE Initialized by $\bm{A}$, learn $\{\bm{U},\bm{A}\}$ from $\mathcal{S}$ by Algorithm~\ref{alg1}.
   \STATE Initialized by $\bm{U}$, obtain new $\mathcal{S}'$ by Algorithm~\ref{alg2}.
   \STATE \textbf{until} {Convergence}
\end{algorithmic}
\end{algorithm}

Combining the diversity-driven superposition method with our stochastic optimization algorithm, we propose Algorithm~\ref{alg3} to learn multi-agent Hawkes process. 
In each iteration, the proposed algorithm first applies the event-level stochastic optimization algorithm for the superposed sequences, helping to improve estimation of endogenous impact tensor $\bm{A}$. 
The exogenous intensities are then optimized by applying our stochastic optimization algorithm for the original sequences, and the endogenous impact tensor is finetuned simultaneously. 
Finally, given a better estimation of exogenous intensities, the diversity-driven superposition is applied again to further suppress the upper bound of the parameters for superposed Hawkes processes. 
According to our theoretical analysis, with high probability this new algorithm can achieve learning results with tighter risk bounds. 

\subsection{The cost of superposition}
When we apply the superposition-assisted learning algorithm, we may increase the computational complexity. 
Even if we just run one epoch ($i.e.$, traversing all events once) in each iteration (lines 5 and 6 of Algorithm~\ref{alg3}), the computational operations are doubled because we no longer learn $\bm{A}$ and $\bm{U}$ simultaneously. 
Further, when learning from a $K$-nonaugmented superposition, the density of an event is $K$ times as dense as the original density. 
This means that we need to consider more historical events when calculating the intensities and the gradients. 
As a result, the computational complexity per iteration becomes $\mathcal{O}(BJK)$, where $JK$ is the number of historical events required in the superposed case. 
Fortunately, in the following experiments we show that our algorithm can outperform existing methods with respect to convergence. 
Therefore, the increase of the computational complexity per iteration may be ignorable because fewer iterations are required.

\section{Related Work}\label{sec:relate}
\subsection{Learning Hawkes processes}
Hawkes processes~\citep{hawkes1971point} are useful tools for modeling real-world event sequences. 
Recently, many variants of Hawkes processes have been proposed, $e.g.$, the mixture of Hawkes processes~\citep{yang2013mixture}, the nonlinear Hawkes process~\citep{wang2016isotonic} and the locally-stationary Hawkes process~\citep{roueff2016locally,xu17b}. 
These approaches show potential for many problems, such as network analysis~\citep{zhao2015seismic} and quantitative finance~\citep{bacry2012non,hardiman2013critical}. 
Focusing on recommendation systems, Hawkes process-based methods~\citep{du2015time,wang2016coevolutionary} achieve encouraging results. 

Maximum likelihood estimation (MLE) is one of the most popular approaches for learning Hawkes processes~\citep{lewis2011nonparametric,zhou2013learning}. 
Besides MLE, least squares methods~\citep{eichler2017graphical}, Wiener-Hopf equations~\citep{bacry2012non} and the Cumulants-based method~\citep{achab2016uncovering} are also applied. 
However, stochastic optimization for Hawkes processes has not been studied systematically. 
For nonparametric Hawkes processes, online learning methods were proposed in~\citep{hall2016tracking,yang2017online}, but they use time-consuming discretization or kernel-estimation when learning models, and thus have poor scalability. 
The risk bound and the sample complexity of learning a single Hawkes process are investigated in~\citep{gomez2015estimating,bacry2015generalization,yang2017online}.
However, this work does not consider the risk bound when learning multiple Hawkes processes assisted with superposition.

\subsection{Superposed point processes}
The early work in~\citep{cox1954superposition} studied the superposition of renewal processes and applied its property to analyze pooling signals in neurophysiology. 
The work in~\citep{cinlar1968superposition,ccinlar1968superposition} analyzed the independence of source processes and the dynamics of the source indicator, for the superposition of Poisson processes and that of renewal processes. 
The superposition of arrival processes has been applied to model queue behaviors~\citep{albin1984approximating} and analyze voice data~\citep{sriram1986characterizing}. 
The work in~\citep{moller2012transforming} proved that a spatial point process can be transformed to a spatial Poisson process by superposing its observations randomly. 
The MCMC sampling~\citep{redenbach2015classification} and variational inference~\citep{rajala2016variational} are proposed for classifying different spatial processes from their superposed observations~\citep{walsh2005classification}. 
However, the work above did not investigate the superposition of complicated point processes like Hawkes processes. 
Recently \citep{xu2017benefits} started to investigate the usefulness of superposed Hawkes processes in the framework of least-squares estimation. 
However, for Hawkes processes, least-squares estimation is often inferior to maximum likelihood estimation with respect to sample complexity and convergence.

\begin{figure*}[t]
\centering
\subfigure[]{
\includegraphics[height=0.17\linewidth]{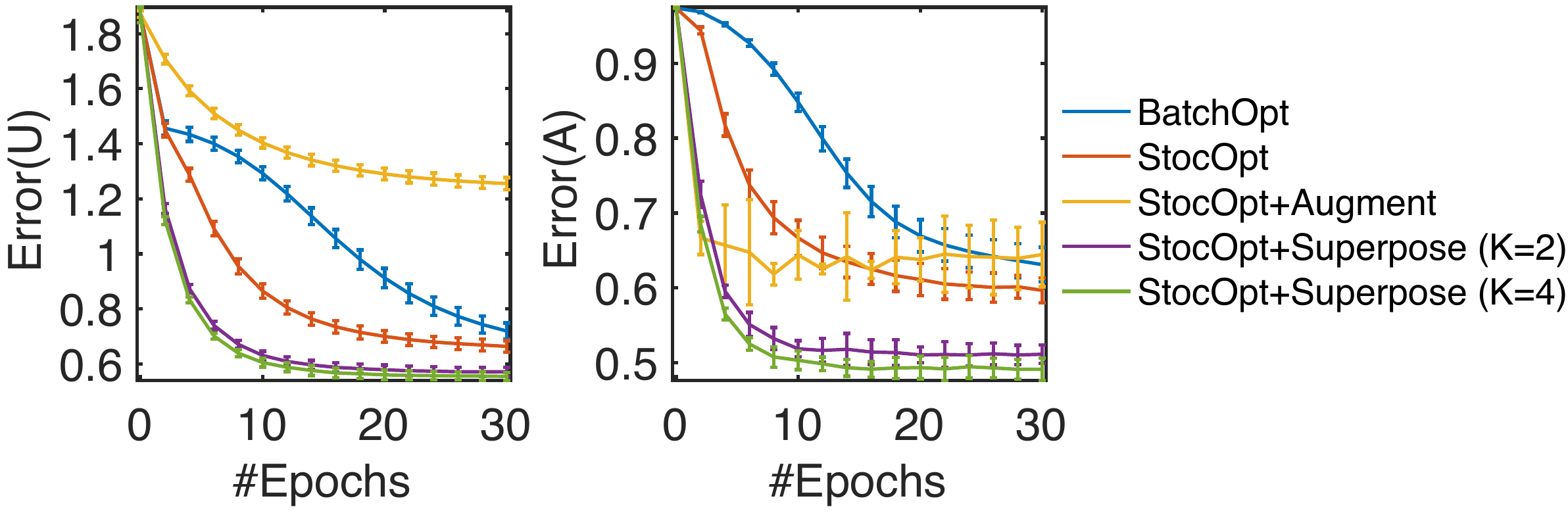}\label{fig:mle1}
}
\subfigure[]{
\includegraphics[height=0.17\linewidth]{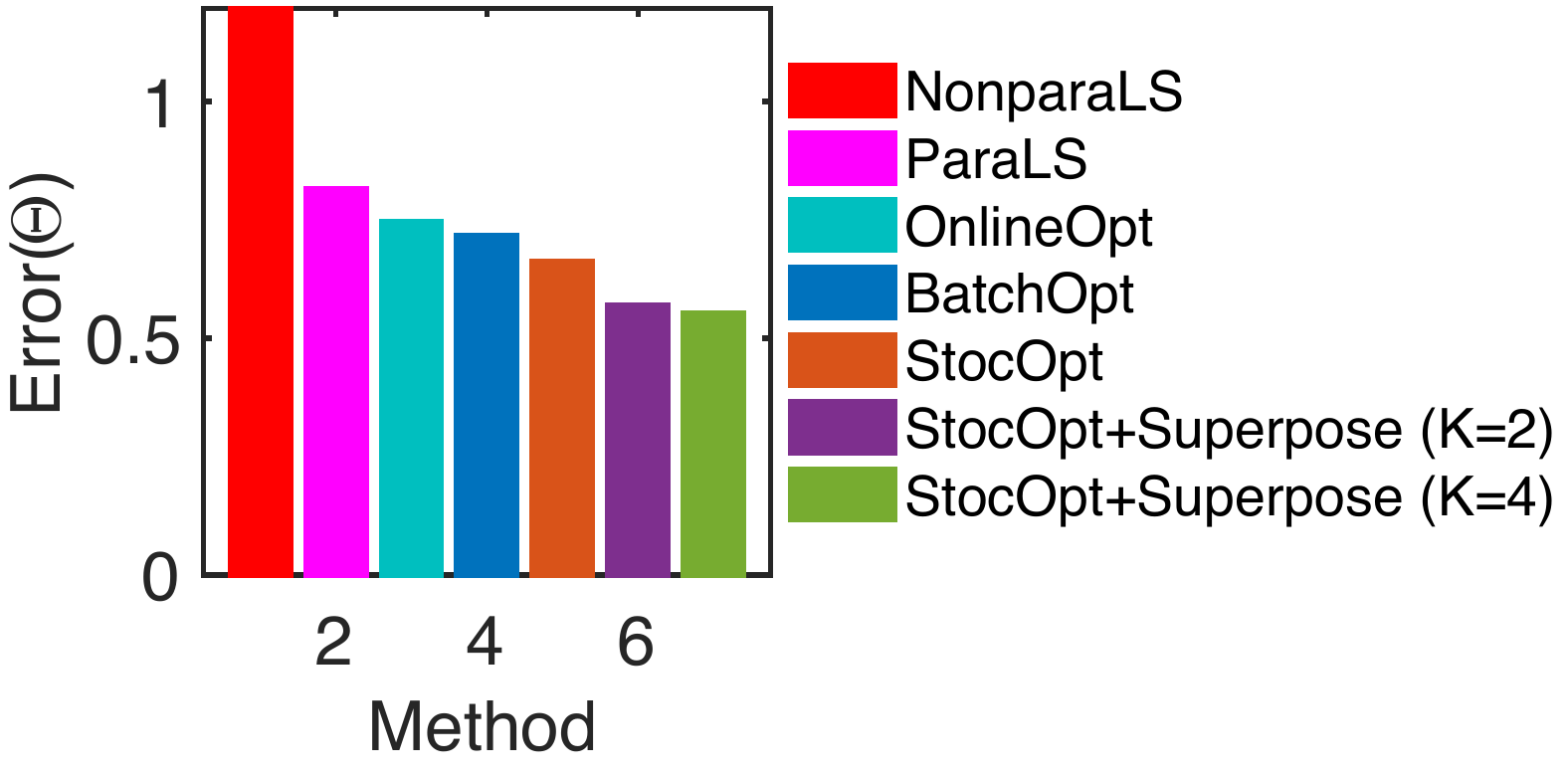}\label{fig:mle2}
}\vspace{-10pt}
\caption{Comparisons on estimation errors for various methods listed in the figure.}\label{fig:mle}
\end{figure*}

\section{Experiments}
\subsection{Validations based on synthetic data}
To verify the usefulness of our superposition-assisted stochastic optimization method, we first test it on a synthetic data set and compare it with alternative approaches. 
To generate the synthetic data, we define a multi-agent Hawkes process with $C=20$ entities and $M=100$ agents. 
The exogenous intensity matrix $\bm{U}\in\mathbb{R}^{C\times M}$ is a random matrix, whose values are uniformly sampled from the interval $[0, \frac{1}{C}]$. 
For the endogenous impact functions, we use a single exponential function $\exp(-wt)$ as the base, where $w=1$. 
Therefore, the endogenous impact tensor degrades to a matrix $\bm{A}\in\mathbb{R}^{C\times C}$, and we set $\bm{A}$ to be a random matrix with spectral norm $\|\bm{A}\|_2=0.7$. 
Using this model, for each agent we generate an sequence with at most $100$ events in the time window $[0, 50]$, using the branching process-based simulation method~\citep{moller2006approximate}.

To learn a multi-agent Hawkes process from the data, we consider the following four strategies:

\vspace{-4pt}
\textbf{1. BatchOpt:} Learn the model directly by traditional batch optimization~\citep{zhou2013learning}.

\vspace{-4pt}
\textbf{2. StocOpt:} Learn the model directly by Algorithm~\ref{alg1}.

\vspace{-4pt}
\textbf{3. StocOpt + Augment:} Use Algorithm~\ref{alg2} to generate $M'$ superposed sequences, add them to original data and then learn a $(M+M')$-agent Hawkes process by Algorithm~\ref{alg1}.

\vspace{-4pt}
\textbf{4. StocOpt + Superpose:} Learn the model by Algorithm~\ref{alg3}.

Besides the proposed method and its variants mentioned above, we further consider three additional methods as references against which to compare: the least squares method for nonparametric Hawkes process (\textbf{NonparaLS}) in~\citep{eichler2017graphical}, the least squares method for parametric Hawkes process (\textbf{ParaLS}) in~\citep{xu2017benefits}, and the online learning method (\textbf{OnlineOpt}) in~\citep{yang2017online}.

We test each method in $10$ trials.
In each trial the exogenous intensity matrix is initialized by $[\frac{N_c^m(T)}{T}]$ and the endogenous impact tensor is initialized randomly. 
For our method and its variant, the averaged relative estimation errors of the exogenous matrix and the endogenous impact tensor ($i.e.$, $\frac{\|\bm{U}_*-\widehat{\bm{U}}\|_F}{\|\bm{U}_*\|_F}$ and $\frac{\|\bm{A}_*-\widehat{\bm{A}}\|_F}{\|\bm{A}_*\|_F}$) and their standard deviations are recorded with respect to the increase of epochs, which are shonw in Fig.~\ref{fig:mle1}. 
For a fair comparison, the impact functions learned by ``NonparaLS'', ``ParaLS'' and ``OnlineOpt'' are further fitted with the ground truth of decay kernel, and then the tensor $\widehat{\bm{A}}$ are estimated accordingly. 
We just compare their averaged final estimation error ($\frac{\|\bm{\theta}_*-\hat{\bm{\theta}}\|_2}{\|\bm{\theta}_*\|_2}$) with ours in Fig.~\ref{fig:mle2} because they do not run epochs like our method.
In the experiment, the $K$-nonaugmented superposition for ``StocOpt + Augment'' is with $K=2$ and that for ``StocOpt + Superpose'' is with $K=2$ and $4$, respectively. 

The experimental results in Fig.~\ref{fig:mle1} verify our theoretical analysis.
We find that the our event-level stochastic optimization algorithm outperforms a traditional batch optimization algorithm on convergence. 
With the help of $K$-nonaugmented superposition, the proposed ``StocOpt + Superpose'' method achieves the best performance consistently with respect to both convergence and estimation errors. 
Fig.~\ref{fig:mle2} shows that our method also obtains superior learning results relative to other existing methods. 
In both Fig.~\ref{fig:mle1} and Fig.~\ref{fig:mle2}, with the increase of $K$, the estimation errors of the proposed method are further reduced, especially for the endogenous impact tensor.
This means that when we increase the number of sequences that are superposed together the dimension of parameters is reduced greatly while their bound just increases slightly, and thus the risk bound of our learning result is further tightened. 

Additionally, the ``StocOpt + Augment'' method does not work well when applied to multi-agent Hawkes process, which can be explained by our theory.
In particular, the ``StocOpt + Augment'' method augments the original data by $M'$ more superposed sequences. 
Such a data-augmentation strategy may be helpful for learning a single Hawkes process, as implied in~\citep{xu2017benefits}, but it does harm for learning the multi-agent Hawkes process because it introduces more agents (parameters).
As a result, both the upper bound and the dimension of parameters in the new model are larger than those of original one.
Accordingly, inequality (\ref{cond}) does not hold and the risk bound increases.

\subsection{Applications to cold-start recommendations}
Our multi-agent Hawkes process is well suited for analysis of online shopping behavior. 
The agents and the entities are respectively the users and the items, and each event sequence corresponds to a shopping history of a user. 
In our model, the relationships among items are described by the endogenous impact tensor $\bm{A}$ while users' preferences on those items are captured by the exogenous intensity matrix $\bm{U}$.
Ideally, by learning $\bm{A}$ and $\bm{U}$ we can recommend items for a user $m$ at time $t$ by estimating the maximum $\lambda_c^m(t)$. 
In the cold-start problem, however, the event sequences are extremely short, so we can't learn reliable preferences for users or triggering patterns for items. 

We imitate a cold-start situation by selecting training and testing data from the Amazon product data set (APD)~\citep{he2016ups}. 
Considering items in $24$ categories, we focus on the users having few purchases and recommend items for them. 
For each category, we find the items having more than $40$ rating behaviors and select their users satisfying the following three conditions: 1) the number of  their shopping behaviors from January 2014 to April 2014 is no more than five; 2) the scores they gave to these items are $4$ or $5$; and 3) they bought at least one item from April 2014 to July 2014. 
Based on the short shopping sequences from January to April, we aim to predict/recommend items for the users from April to July.

\begin{table*}[t]
\small
\centering
\caption{Summary of the Top-5 performance for various methods.}\label{tab1}
\vspace{3pt}
\setlength{\tabcolsep}{0.7pt}
\begin{tabular}
{c|
@{\hspace{2pt}}c@{\hspace{3pt}}c@{\hspace{3pt}}c@{\hspace{2pt}}|
@{\hspace{2pt}}c@{\hspace{3pt}}c@{\hspace{3pt}}c@{\hspace{2pt}}|
@{\hspace{2pt}}c@{\hspace{3pt}}c@{\hspace{3pt}}c@{\hspace{2pt}}|
@{\hspace{2pt}}c@{\hspace{3pt}}c@{\hspace{3pt}}c@{\hspace{2pt}}|
@{\hspace{2pt}}c@{\hspace{3pt}}c@{\hspace{3pt}}c@{\hspace{2pt}}|
@{\hspace{2pt}}c@{\hspace{3pt}}c@{\hspace{3pt}}c@{\hspace{1pt}}} 
\hline\hline
Method 
&\multicolumn{3}{c|@{\hspace{2pt}}}{SLIM}
&\multicolumn{3}{c|@{\hspace{2pt}}}{BPR}
&\multicolumn{3}{c|@{\hspace{2pt}}}{FPMC} 
&\multicolumn{3}{c|@{\hspace{2pt}}}{Single HP}
&\multicolumn{3}{c|@{\hspace{2pt}}}{MHPs} 
&\multicolumn{3}{c}{MHPs+Superpose}\\ \hline
Metric 
&$P@5$ &$R@5$ &$F_1@5$
&$P@5$ &$R@5$ &$F_1@5$ 
&$P@5$ &$R@5$ &$F_1@5$
&$P@5$ &$R@5$ &$F_1@5$ 
&$P@5$ &$R@5$ &$F_1@5$
&$P@5$ &$R@5$ &$F_1@5$\\ \hline
Instant Video 
&\textbf{2.37}  &\textbf{11.84} &\textbf{3.95}
&1.39  &5.79  &2.19  
&1.05  &5.23  &1.74  
&1.83  &9.14  &3.05  
&2.15  &10.74  &3.58  
&2.31  &11.57  &3.79\\ 
Android App 
&0.96  &4.81  &1.60 
&\textbf{1.48}  &\textbf{6.11}  &\textbf{2.32}  
&1.13  &5.63  &1.88  
&0.68  &3.41  &1.14  
&0.92  &4.60  &1.53  
&1.11  &5.56  &1.85\\ 
Automotive 
&0.70  &3.53  &1.17 
&\textbf{0.93}  &\textbf{4.08}  &\textbf{1.49}  
&0.69  &3.46  &1.15  
&0.31  &1.55  &0.52  
&0.31  &1.55  &0.52  
&0.71  &3.53  &1.18\\ 
Baby  
&0.26  &1.24  &0.42
&0.31  &1.53  &0.51  
&\textbf{0.48}  &\textbf{2.42}  &\textbf{0.81}  
&0.33  &1.67  &0.56  
&0.25  &1.25  &0.42  
&0.32  &1.59  &0.53\\ 
Beauty 
&1.25  &6.24  &2.08
&1.13  &1.43  &1.02  
&0.55  &2.75  &0.92  
&1.26  &6.29  &2.10  
&1.27  &6.33  &2.11  
&\textbf{1.28}  &\textbf{6.37}  &\textbf{2.12}\\ 
Book
&0.32  &1.61  &0.54
&0.33  &1.68  &0.59
&0.30  &1.52  &0.50
&0.47  &2.35  &0.77
&0.46  &2.32  &0.78
&\textbf{0.49}  &\textbf{2.45}  &\textbf{0.82}\\
CDs, Vinyl  
&0.44  &2.23  &0.74
&0.39  &1.67  &0.62  
&0.52  &2.60  &0.87  
&0.47  &2.37  &0.79  
&0.39  &1.95  &0.65  
&\textbf{0.64}  &\textbf{3.20}  &\textbf{1.07}\\ 
Cell Phone 
&0.96  &4.81  &1.61 
&1.03  &3.78  &1.56  
&\textbf{1.04}  &\textbf{5.20}  &\textbf{1.73}  
&0.84  &4.22  &1.41  
&0.73  &3.67  &1.22  
&0.71  &3.55  &1.18\\ 
Clothes, Jewelry 
&0.12  &0.63  &0.21 
&\textbf{0.26}  &\textbf{1.01}  &\textbf{0.40}  
&0.22  &1.10  &0.37  
&0.09  &0.45  &0.15  
&0.08  &0.41  &0.14  
&0.10  &0.48  &0.16\\ 
Digital Music  
&0.98  &4.92  &1.64
&\textbf{1.97}  &\textbf{8.06}  &\textbf{3.09}  
&1.59  &7.95  &2.65  
&0.66  &3.28  &1.09  
&0.66  &3.28  &1.09  
&1.64  &8.20  &2.73\\
Electronics
&0.23  &1.17  &0.39
&0.24  &1.20  &0.42
&0.20  &1.02  &0.33
&0.27  &1.38  &0.46
&0.27  &1.37  &0.45
&\textbf{0.28} &\textbf{1.40} &\textbf{0.47}\\
Grocery, food
&0.93  &4.67  &1.56
&0.83  &2.32  &1.15  
&0.67  &3.34  &1.11  
&0.76  &3.81  &1.27  
&0.91  &4.54  &1.51  
&\textbf{0.96}  &\textbf{4.80}  &\textbf{1.60}\\ 
Health Care  
&0.80  &3.99  &1.33
&0.23  &0.61  &0.29  
&0.26  &1.28  &0.43  
&0.79  &3.94  &1.31  
&0.82  &4.08  &1.36  
&\textbf{0.85}  &\textbf{4.24}  &\textbf{1.41}\\ 
Home, Kitchen 
&0.43  &2.21  &0.73 
&0.17  &0.47  &0.22  
&0.12  &0.59  &0.20  
&0.36  &1.80  &0.60  
&0.43  &2.13  &0.71  
&\textbf{0.44}  &\textbf{2.18}  &\textbf{0.74}\\ 
Kindle Store  
&0.99  &4.93  &1.65
&0.26  &0.52  &0.30  
&0.14  &0.72  &0.24  
&1.00  &5.02  &1.67  
&1.00  &5.02  &1.67  
&\textbf{1.08}  &\textbf{5.38}  &\textbf{1.79}\\ 
Movie, TV  
&1.20  &6.04  &2.01
&0.92  &3.37  &1.36  
&0.73  &3.64  &1.21  
&1.20  &5.98  &1.99  
&1.21  &6.05  &2.02  
&\textbf{1.24}  &\textbf{6.20}  &\textbf{2.07}\\ 
Music Instrument
&0.49  &2.47  &0.82  
&\textbf{2.22}  &\textbf{10.29}  &\textbf{3.60}  
&1.69  &8.47  &2.82  
&0.25  &1.23  &0.41  
&0.25  &1.23  &0.41  
&1.73  &8.64  &2.88\\ 
Office Product  
&0.64  &3.20  &1.07
&0.80  &2.18  &1.07  
&0.42  &2.11  &0.70  
&0.74  &3.72  &1.24  
&0.73  &3.65  &1.22  
&\textbf{0.83}  &\textbf{4.17}  &\textbf{1.39}\\ 
Patio Lawn 
&0.56  &2.86  &0.95 
&0.51  &2.10  &0.80  
&0.34  &1.72  &0.57  
&0.57  &2.87  &0.96  
&0.45  &2.23  &0.74  
&\textbf{0.67}  &\textbf{3.34}  &\textbf{1.11}\\ 
Pet Supply  
&0.86  &4.29  &1.43
&0.78  &2.25  &1.08  
&0.75  &3.75  &1.25  
&0.79  &3.94  &1.31  
&0.83  &4.14  &1.38  
&\textbf{0.91}  &\textbf{4.54}  &\textbf{1.51}\\ 
Sports 
&0.35  &1.76  &0.59 
&0.48  &1.79  &0.72  
&0.42  &2.08  &0.69  
&0.41  &2.00  &0.67  
&0.42  &2.00  &0.68  
&\textbf{0.46}  &\textbf{2.32}  &\textbf{0.77}\\ 
Home Tool  
&0.18  &0.88  &0.29
&0.60  &2.30  &0.91  
&0.67  &3.36  &1.12  
&0.19  &0.95  &0.32  
&0.20  &1.02  &0.34  
&\textbf{0.69}  &\textbf{3.46}  &\textbf{1.15}\\ 
Toys, Game 
&1.16  &5.83  &1.93
&0.69  &2.98  &1.10  
&0.58  &2.89  &0.96  
&1.04  &5.19  &1.73  
&1.13  &5.65  &1.88  
&\textbf{1.17}  &\textbf{5.83}  &\textbf{1.94}\\ 
Video Game 
&0.90  &4.52  &1.51 
&1.05  &5.23  &1.74  
&0.85  &4.23  &1.41  
&0.59  &2.93  &0.98  
&0.67  &3.33  &1.11  
&\textbf{1.17}  &\textbf{4.28}  &\textbf{1.77}\\ \hline
Overall
&0.76  &3.78  &1.26
&0.79  &3.04  &1.19	
&0.64  &3.21  &1.07	
&0.66  &3.31  &1.11	
&0.69  &3.44  &1.14	
&\textbf{0.90}  &\textbf{4.45}  &\textbf{1.54}\\
\hline\hline
\end{tabular}\label{tab:result}
\end{table*}

We address the cold-start problem by enhancing the robustness of learning results with the help of our superposition-assisted stochastic optimization algorithm. 
To demonstrate the importance of superposition, we learn a multi-agent Hawkes process from the original data (\textbf{MHPs}) and from its $K$-nonaugmented superposition (proposed \textbf{MHPs + Superpose}), respectively. 
{Furthermore, we consider four different models and methods, including learning a single Hawkes process from the original data (\textbf{Single HP})~\citep{zhou2013learning}, the sparse linear method (\textbf{SLIM})~\citep{ning2011slim}, the Bayesian personalized ranking (\textbf{BPR})~\citep{rendle2009bpr} and the factorization of personalized Markov chains (\textbf{FPMC})~\citep{rendle2010factorizing}. 
``SLIM'' is one of the state-of-the-art item-to-item recommendation methods~\citep{christakopoulou2016local,zheng2014cslim}, which outperforms most existing collaborative filtering methods, while ``BPR'' and ``FPMC'' are classical sequential recommendation methods. 
They are representative and competitive baselines for our Hawkes process-based methods.}
For Hawkes process-based methods, we recommend the item with the highest endogenous intensity for each user:
\begin{eqnarray*}
\begin{aligned}
c_{next}=\arg\max_{c\in\mathcal{C}}\sideset{}{_{(t_i^m,c_i^m)\in\mathcal{H}_t^m}}\sum \phi_{cc_i}(t-t_i).
\end{aligned}
\end{eqnarray*}
Here we do not consider exogenous intensities because the events for each individual are insufficient, and thus the learned exogenous intensities are unreliable.

Table~\ref{tab1} summarizes the performance of the various methods. 
For each method, we generate a recommendation list $\bm{r}^m$ with the top-5 most possible items for each user. 
Given the real set of the items that user $m$ will buy, denoted as $\bm{t}^m$, we use the top-5 precision ($P@5$), recall ($R@5$) and $F_{1}$-score ($F_1@5$) for evaluation. 
Their definitions are in the Supplementary Material.
The proposed model and learning method ``MHPs + Superpose'' improves the performance of recommendation system for most categories and obtains the best overall performance. 
These results suggest the potential of our method to solve the cold-start problem of recommendation systems. 
In those failed cases, we can find that all Hawkes-process-based methods are inferior to ``BPR'' or ``FPMC''. 
We believe this implies that for those item categories the shopping behaviors may not match well with Hawkes processes and the failures are due to model misspecification. 
However, even in those cases, our ``MHPs + Superpose'' method still outperforms the remaining two Hawkes-process-based methods except the ``Cell Phone'' category, which demonstrates the superiority of or method as well.

Note that for different item categories, we choose the $K$-nonaugmented superposition with different $K$ to obtain optimal performance. 
Figure~\ref{fig:coldstart} labels the optimal $K$ for different categories, and we find generally that the data with a large number of items and users require few superposition operations (small $K$). 
A potential reason for this is that for large-scale data, the endogenous impact tensor should be very sparse. 
Superposing too many sequences together is likely to introduce many nonexistent triggering patterns among unrelated items, and cause model misspecification. 
More experimental results and detailed analysis are given in the Supplementary Material.

\begin{figure}[t]
\centering
\includegraphics[width=0.7\linewidth]{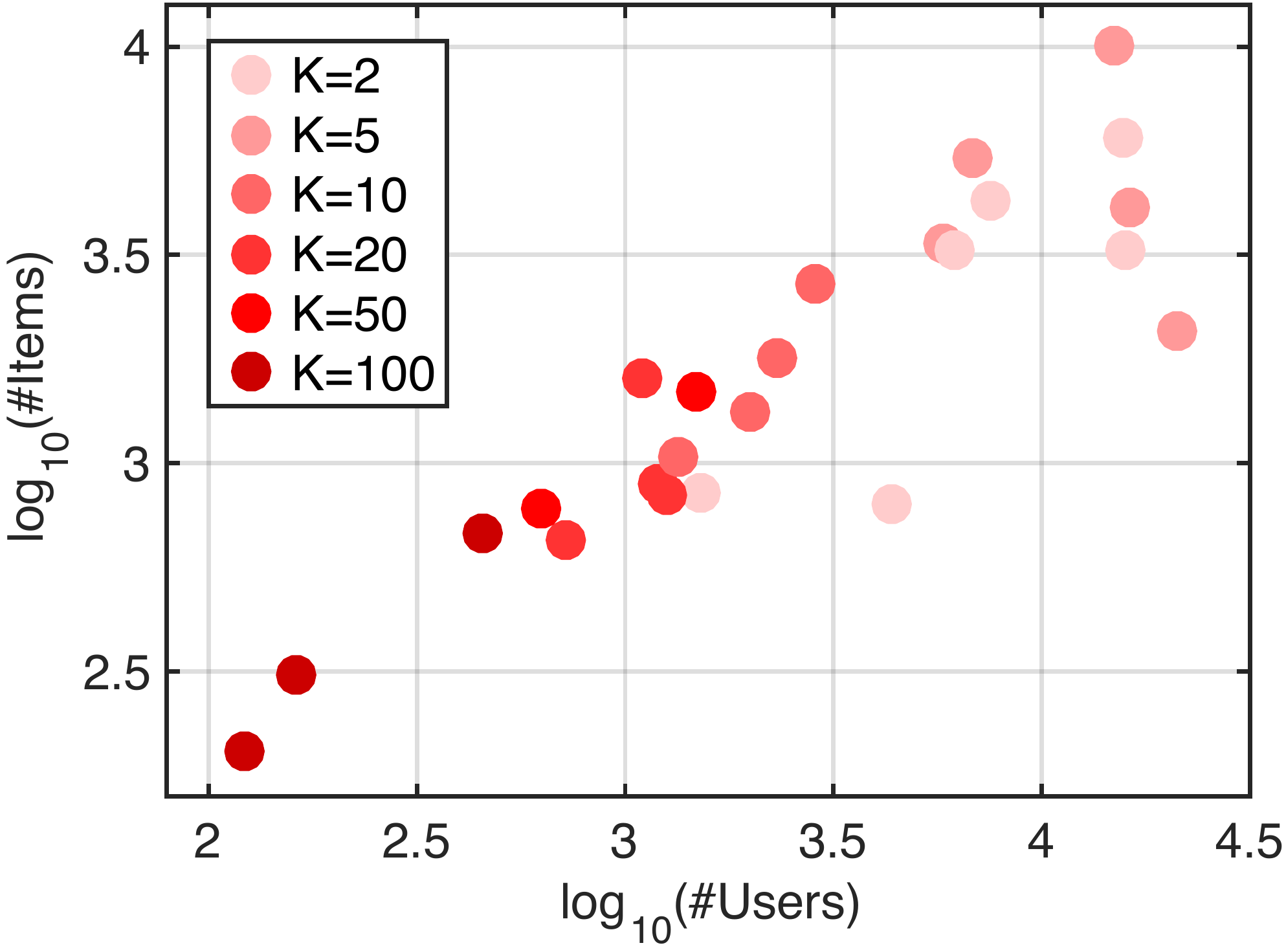}\vspace{-10pt}
\caption{{Each point in the plot corresponds to an item category in the data. 
Its coordinate indicates how many users ($x$-coordinate) and items ($y$-coordinate) the category has. 
For different categories, we apply the $K$-nonaugmented superposition with different $K$'s. 
For each category, the color of its point represents the $K$ corresponding to its best recommendation result shown in Table~\ref{tab1}.}
}\label{fig:coldstart}
\end{figure}

\section{Conclusions}
We have proposed a stochastic optimization algorithm for learning multi-agent Hawkes processes, and have explored the influence of the superposition operation on the learning results. 
We demonstrate that with the help of superposition we can learn the proposed model with lower excess risk. 
We verify our theoretical results on synthetic data, and on real-world data show the encouraging performance on solving the cold-start problem of recommendation systems.

\bibliography{example_paper}

\begin{thebibliography}{44}
\providecommand{\natexlab}[1]{#1}
\providecommand{\url}[1]{\texttt{#1}}
\expandafter\ifx\csname urlstyle\endcsname\relax
  \providecommand{\doi}[1]{doi: #1}\else
  \providecommand{\doi}{doi: \begingroup \urlstyle{rm}\Url}\fi

\bibitem[Achab et~al.(2016)Achab, Bacry, Ga{\"\i}ffas, Mastromatteo, and
  Muzy]{achab2016uncovering}
Achab, Massil, Bacry, Emmanuel, Ga{\"\i}ffas, St{\'e}phane, Mastromatteo,
  Iacopo, and Muzy, Jean-Francois.
\newblock Uncovering {C}ausality from multivariate {H}awkes integrated
  {C}umulants.
\newblock \emph{arXiv preprint arXiv:1607.06333}, 2016.

\bibitem[Albin(1984)]{albin1984approximating}
Albin, Susan~L.
\newblock Approximating a point process by a renewal process, {II}:
  {S}uperposition arrival processes to queues.
\newblock \emph{Operations Research}, 32\penalty0 (5):\penalty0 1133--1162,
  1984.

\bibitem[Bacry \& Muzy(2014)Bacry and Muzy]{bacry2014second}
Bacry, Emmanuel and Muzy, Jean-Francois.
\newblock Second order statistics characterization of {H}awkes processes and
  non-parametric estimation.
\newblock \emph{arXiv preprint arXiv:1401.0903}, 2014.

\bibitem[Bacry et~al.(2012)Bacry, Dayri, and Muzy]{bacry2012non}
Bacry, Emmanuel, Dayri, Khalil, and Muzy, Jean-Fran{\c{c}}ois.
\newblock Non-parametric kernel estimation for symmetric {H}awkes processes:
  {A}pplication to high frequency financial data.
\newblock \emph{The European Physical Journal B-Condensed Matter and Complex
  Systems}, 85\penalty0 (5):\penalty0 1--12, 2012.

\bibitem[Bacry et~al.(2015)Bacry, Ga{\"\i}ffas, and
  Muzy]{bacry2015generalization}
Bacry, Emmanuel, Ga{\"\i}ffas, St{\'e}phane, and Muzy, Jean-Fran{\c{c}}ois.
\newblock A generalization error bound for sparse and low-rank multivariate
  {H}awkes processes.
\newblock \emph{arXiv preprint arXiv:1501.00725}, 2015.

\bibitem[Bertsekas \& Scientific(2015)Bertsekas and
  Scientific]{bertsekas2015convex}
Bertsekas, Dimitri~P and Scientific, Athena.
\newblock \emph{Convex optimization algorithms}.
\newblock Athena Scientific Belmont, 2015.

\bibitem[Christakopoulou \& Karypis(2016)Christakopoulou and
  Karypis]{christakopoulou2016local}
Christakopoulou, Evangelia and Karypis, George.
\newblock Local item-item models for {T}op-${N}$ recommendation.
\newblock In \emph{RecSys}, 2016.

\bibitem[{\c{C}}inlar(1968)]{ccinlar1968superposition}
{\c{C}}inlar, Erhan.
\newblock On the superposition of m-dimensional point processes.
\newblock \emph{Journal of Applied Probability}, 5\penalty0 (1):\penalty0
  169--176, 1968.

\bibitem[{\c{C}}inlar \& Agnew(1968){\c{C}}inlar and
  Agnew]{cinlar1968superposition}
{\c{C}}inlar, Erhan and Agnew, RA.
\newblock On the superposition of point processes.
\newblock \emph{Journal of the Royal Statistical Society. Series B
  (Methodological)}, pp.\  576--581, 1968.

\bibitem[Cox \& Smith(1954)Cox and Smith]{cox1954superposition}
Cox, DR and Smith, Walter~L.
\newblock On the superposition of renewal processes.
\newblock \emph{Biometrika}, 41\penalty0 (1-2):\penalty0 91--99, 1954.

\bibitem[Daley \& Vere-Jones(2007)Daley and Vere-Jones]{daley2007introduction}
Daley, Daryl~J and Vere-Jones, David.
\newblock \emph{An introduction to the theory of point processes: volume {II}:
  general theory and structure}, volume~2.
\newblock Springer Science \& Business Media, 2007.

\bibitem[Du et~al.(2015)Du, Wang, He, Sun, and Song]{du2015time}
Du, Nan, Wang, Yichen, He, Niao, Sun, Jimeng, and Song, Le.
\newblock Time-sensitive recommendation from recurrent user activities.
\newblock In \emph{NIPS}, 2015.

\bibitem[Eichler et~al.(2017)Eichler, Dahlhaus, and
  Dueck]{eichler2017graphical}
Eichler, Michael, Dahlhaus, Rainer, and Dueck, Johannes.
\newblock Graphical modeling for multivariate {H}awkes processes with
  nonparametric link functions.
\newblock \emph{Journal of Time Series Analysis}, 38\penalty0 (2):\penalty0
  225--242, 2017.

\bibitem[Farajtabar et~al.(2014)Farajtabar, Du, Rodriguez, Valera, Zha, and
  Song]{farajtabar2014shaping}
Farajtabar, Mehrdad, Du, Nan, Rodriguez, Manuel~Gomez, Valera, Isabel, Zha,
  Hongyuan, and Song, Le.
\newblock Shaping social activity by incentivizing users.
\newblock In \emph{NIPS}, 2014.

\bibitem[Gomez-Rodriguez et~al.(2015)Gomez-Rodriguez, Song, Daneshmand, and
  Sch{\"o}lkopf]{gomez2015estimating}
Gomez-Rodriguez, Manuel, Song, Le, Daneshmand, Hadi, and Sch{\"o}lkopf,
  Bernhard.
\newblock Estimating diffusion networks: {R}ecovery conditions, sample
  complexity \& soft-thresholding algorithm.
\newblock \emph{Journal of Machine Learning Research}, 2015.

\bibitem[Hall \& Willett(2016)Hall and Willett]{hall2016tracking}
Hall, Eric~C and Willett, Rebecca~M.
\newblock Tracking dynamic point processes on networks.
\newblock \emph{IEEE Transactions on Information Theory}, 62\penalty0
  (7):\penalty0 4327--4346, 2016.

\bibitem[Hardiman et~al.(2013)Hardiman, Bercot, and
  Bouchaud]{hardiman2013critical}
Hardiman, Stephen, Bercot, Nicolas, and Bouchaud, Jean-Philippe.
\newblock Critical reflexivity in financial markets: a {H}awkes process
  analysis.
\newblock \emph{The European Physical Journal B: Condensed Matter and Complex
  Systems}, 86\penalty0 (10):\penalty0 1--9, 2013.

\bibitem[Hawkes(1971)]{hawkes1971point}
Hawkes, Alan~G.
\newblock Point spectra of some mutually exciting point processes.
\newblock \emph{Journal of the Royal Statistical Society. Series B
  (Methodological)}, pp.\  438--443, 1971.

\bibitem[He \& McAuley(2016)He and McAuley]{he2016ups}
He, Ruining and McAuley, Julian.
\newblock Ups and downs: Modeling the visual evolution of fashion trends with
  one-class collaborative filtering.
\newblock In \emph{WWW}, 2016.

\bibitem[Lewis \& Mohler(2011)Lewis and Mohler]{lewis2011nonparametric}
Lewis, Erik and Mohler, George.
\newblock A nonparametric {EM} algorithm for multiscale {H}awkes processes.
\newblock \emph{Journal of Nonparametric Statistics}, 1\penalty0 (1):\penalty0
  1--20, 2011.

\bibitem[Luo et~al.(2015)Luo, Xu, Zhen, Ning, Zha, Yang, and
  Zhang]{luo2015multi}
Luo, Dixin, Xu, Hongteng, Zhen, Yi, Ning, Xia, Zha, Hongyuan, Yang, Xiaokang,
  and Zhang, Wenjun.
\newblock Multi-task multi-dimensional {H}awkes processes for modeling event
  sequences.
\newblock In \emph{IJCAI}, 2015.

\bibitem[M{\o}ller \& Berthelsen(2012)M{\o}ller and
  Berthelsen]{moller2012transforming}
M{\o}ller, Jesper and Berthelsen, Kasper~K.
\newblock Transforming spatial point processes into poisson processes using
  random superposition.
\newblock \emph{Advances in Applied Probability}, 44\penalty0 (1):\penalty0
  42--62, 2012.

\bibitem[M{\o}ller \& Rasmussen(2006)M{\o}ller and
  Rasmussen]{moller2006approximate}
M{\o}ller, Jesper and Rasmussen, Jakob~G.
\newblock Approximate simulation of {H}awkes processes.
\newblock \emph{Methodology and Computing in Applied Probability}, 8\penalty0
  (1):\penalty0 53--64, 2006.

\bibitem[Ning \& Karypis(2011)Ning and Karypis]{ning2011slim}
Ning, Xia and Karypis, George.
\newblock {SLIM}: {S}parse linear methods for {T}op-${N}$ recommender systems.
\newblock In \emph{ICDM}, 2011.

\bibitem[Rajala et~al.(2016)Rajala, Redenbach, S{\"a}rkk{\"a}, and
  Sormani]{rajala2016variational}
Rajala, Tuomas, Redenbach, Claudia, S{\"a}rkk{\"a}, Aila, and Sormani, Martina.
\newblock Variational {B}ayes approach for classification of points in
  superpositions of point processes.
\newblock \emph{Spatial Statistics}, 15:\penalty0 85--99, 2016.

\bibitem[Redenbach et~al.(2015)Redenbach, S{\"a}rkk{\"a}, and
  Sormani]{redenbach2015classification}
Redenbach, Claudia, S{\"a}rkk{\"a}, Aila, and Sormani, Martina.
\newblock Classification of points in superpositions of strauss and poisson
  processes.
\newblock \emph{Spatial Statistics}, 12:\penalty0 81--95, 2015.

\bibitem[Rendle et~al.(2009)Rendle, Freudenthaler, Gantner, and
  Schmidt-Thieme]{rendle2009bpr}
Rendle, Steffen, Freudenthaler, Christoph, Gantner, Zeno, and Schmidt-Thieme,
  Lars.
\newblock {BPR}: {B}ayesian personalized ranking from implicit feedback.
\newblock In \emph{UAI}, 2009.

\bibitem[Rendle et~al.(2010)Rendle, Freudenthaler, and
  Schmidt-Thieme]{rendle2010factorizing}
Rendle, Steffen, Freudenthaler, Christoph, and Schmidt-Thieme, Lars.
\newblock Factorizing personalized markov chains for next-basket
  recommendation.
\newblock In \emph{WWW}, 2010.

\bibitem[Roueff et~al.(2016)Roueff, Von~Sachs, and
  Sansonnet]{roueff2016locally}
Roueff, Fran{\c{c}}ois, Von~Sachs, Rainer, and Sansonnet, Laure.
\newblock Locally stationary {H}awkes processes.
\newblock \emph{Stochastic Processes and their Applications}, 126\penalty0
  (6):\penalty0 1710--1743, 2016.

\bibitem[Sriram \& Whitt(1986)Sriram and Whitt]{sriram1986characterizing}
Sriram, Kotikalapudi and Whitt, Ward.
\newblock Characterizing superposition arrival processes in packet multiplexers
  for voice and data.
\newblock \emph{Journal on selected areas in communications}, 4\penalty0
  (6):\penalty0 833--846, 1986.

\bibitem[Wachinger \& Golland(2015)Wachinger and Golland]{wachinger2015diverse}
Wachinger, Christian and Golland, Polina.
\newblock Diverse landmark sampling from determinantal point processes for
  scalable manifold learning.
\newblock \emph{arXiv preprint arXiv:1503.03506}, 2015.

\bibitem[Walsh \& Raftery(2005)Walsh and Raftery]{walsh2005classification}
Walsh, Daniel~CI and Raftery, Adrian~E.
\newblock Classification of mixtures of spatial point processes via partial
  bayes factors.
\newblock \emph{Journal of Computational and Graphical Statistics}, 14\penalty0
  (1):\penalty0 139--154, 2005.

\bibitem[Wang et~al.(2016{\natexlab{a}})Wang, Du, Trivedi, and
  Song]{wang2016coevolutionary}
Wang, Yichen, Du, Nan, Trivedi, Rakshit, and Song, Le.
\newblock Coevolutionary latent feature processes for continuous-time user-item
  interactions.
\newblock In \emph{NIPS}, 2016{\natexlab{a}}.

\bibitem[Wang et~al.(2016{\natexlab{b}})Wang, Xie, Du, and
  Song]{wang2016isotonic}
Wang, Yichen, Xie, Bo, Du, Nan, and Song, Le.
\newblock Isotonic {H}awkes processes.
\newblock In \emph{ICML}, 2016{\natexlab{b}}.

\bibitem[Xu et~al.(2016)Xu, Farajtabar, and Zha]{xu2016learning}
Xu, Hongteng, Farajtabar, Mehrdad, and Zha, Hongyuan.
\newblock Learning {G}ranger causality for {H}awkes processes.
\newblock In \emph{ICML}, 2016.

\bibitem[Xu et~al.(2017{\natexlab{a}})Xu, Luo, Chen, and Carin]{xu2017benefits}
Xu, Hongteng, Luo, Dixin, Chen, Xu, and Carin, Lawrence.
\newblock Benefits from superposed hawkes processes.
\newblock \emph{arXiv preprint arXiv:1710.05115}, 2017{\natexlab{a}}.

\bibitem[Xu et~al.(2017{\natexlab{b}})Xu, Luo, and Zha]{xu17b}
Xu, Hongteng, Luo, Dixin, and Zha, Hongyuan.
\newblock Learning {H}awkes processes from short doubly-censored event
  sequences.
\newblock In \emph{ICML}, 2017{\natexlab{b}}.

\bibitem[Yang \& Zha(2013)Yang and Zha]{yang2013mixture}
Yang, Shuang-Hong and Zha, Hongyuan.
\newblock Mixture of mutually exciting processes for viral diffusion.
\newblock In \emph{International Conference on Machine Learning}, pp.\  1--9,
  2013.

\bibitem[Yang et~al.(2017)Yang, Etesami, He, and Kiyavash]{yang2017online}
Yang, Yingxiang, Etesami, Jalal, He, Niao, and Kiyavash, Negar.
\newblock Online learning for multivariate {H}awkes processes.
\newblock In \emph{NIPS}, 2017.

\bibitem[Zhang et~al.(2017)Zhang, Yang, and Jin]{zhang2017empirical}
Zhang, Lijun, Yang, Tianbao, and Jin, Rong.
\newblock Empirical risk minimization for stochastic convex optimization: $ o
  (1/n) $-and $ o (1/n^{2}) $-type of risk bounds.
\newblock In \emph{COLT}, 2017.

\bibitem[Zhao et~al.(2015)Zhao, Erdogdu, He, Rajaraman, and
  Leskovec]{zhao2015seismic}
Zhao, Qingyuan, Erdogdu, Murat~A, He, Hera~Y, Rajaraman, Anand, and Leskovec,
  Jure.
\newblock {SEISMIC}: A self-exciting point process model for predicting tweet
  popularity.
\newblock In \emph{KDD}, 2015.

\bibitem[Zheng et~al.(2014)Zheng, Mobasher, and Burke]{zheng2014cslim}
Zheng, Yong, Mobasher, Bamshad, and Burke, Robin.
\newblock {CSLIM}: {C}ontextual {SLIM} recommendation algorithms.
\newblock In \emph{RecSys}, 2014.

\bibitem[Zhou et~al.(2013)Zhou, Zha, and Song]{zhou2013learning}
Zhou, Ke, Zha, Hongyuan, and Song, Le.
\newblock Learning social infectivity in sparse low-rank networks using
  multi-dimensional {H}awkes processes.
\newblock In \emph{AISTATS}, 2013.

\bibitem[Zhu(2013)]{zhu2013ruin}
Zhu, Lingjiong.
\newblock Ruin probabilities for risk processes with non-stationary arrivals
  and subexponential claims.
\newblock \emph{Insurance: Mathematics and Economics}, 53\penalty0
  (3):\penalty0 544--550, 2013.

\end{thebibliography}
\bibliographystyle{icml2018}

\newpage
\section{Supplementary Material}
\subsection{Proof of Theorem~\ref{thm1}}
\begin{proof}
The heart of the proof is the upper bound on the excess risk of stochastic convex optimization for empirical risk minimization (ERM)~\citep{zhang2017empirical}. 
In particular, for the $D$-dimensional parameters $\bm{\theta}$ learned by minimizing an empirical objective function $\min_{\bm{\theta}\in\Theta}\frac{1}{N}\sum_{n=1}^{N} f_n(\theta)$ from $N$ samples (an approximation of expected function $\mathbb{E}_{f\sim\mathbb{P}}[f(\theta)]$), we denote the optimum parameters and the corresponding object function as $\bm{\theta}_*$ and $F_*$, respectively. 
We assume that
\begin{enumerate}
\item The feasible domain $\Theta$ is a convex set and the parameters are bounded, $i.e.$, $\|\bm{\theta}\|_2^2\leq R$ for $\bm{\theta}\in\Theta$.
\item The random function $f$ is nonnegative, convex, and $\tau$-smooth over $\Theta$, $i.e.$, $\forall~\bm{\theta},\bm{\theta}'\in\Theta,~f\sim\mathbb{P}$, we have
\begin{eqnarray*}
\begin{aligned}
\|\nabla f(\bm{\theta}) - \nabla f(\bm{\theta}')\|_2\leq \tau\|\bm{\theta}-\bm{\theta}'\|_2.
\end{aligned}
\end{eqnarray*}
\item The expected function $F=\mathbb{E}[f]$ is $\gamma$-strongly convex, $i.e.$, $\forall~\bm{\theta},\bm{\theta}'\in\Theta$
\begin{eqnarray*}
\begin{aligned}
F(\bm{\theta})+\nabla F(\bm{\theta})^{\top}(\bm{\theta}'-\bm{\theta})+\frac{\gamma}{2}\|\bm{\theta}'-\bm{\theta}\|_2^2\leq F(\bm{\theta}').
\end{aligned}
\end{eqnarray*}
\end{enumerate}
Then, the Theorem 5 in~\citep{zhang2017empirical} proves that with probability at least $1-2\delta$, $0<\delta<0.5$, we have
\begin{eqnarray}\label{bound0}
\begin{aligned}
&\mathbb{E}[f(\hat{\bm{\theta}})-f(\bm{\theta}_*)]\\
&\leq \frac{4R\tau C_{\epsilon,\delta}}{N}+\frac{4 R\tau^2 C_{\epsilon,\delta}}{\gamma N}\\
&\quad+\frac{4\sqrt{R}\sup_{f\sim \mathbb{P}}\|\nabla f(\bm{\theta}_*)\|_2\log\frac{2}{\delta}}{N}\\
&\quad+\frac{8\tau F_*\log\frac{2}{\delta}}{\gamma N}
+2\sqrt{R}\tau\epsilon\left(2 + \sqrt{\frac{C_{\epsilon,\delta}}{N}} + \frac{C_{\epsilon,\delta}}{N}\right),
\end{aligned}
\end{eqnarray}
where $C_{\epsilon,\delta}=2(\log\frac{2}{\delta}+D\log\frac{6\sqrt{R}}{\epsilon})$ and $\epsilon>0$.

We proof that the learning problem of multi-agent Hawkes process $HP(\bm{U},\bm{A})$ satisfies the conditions mentioned above.
In particular, we have $N=I_{\Sigma}$, and the dimension of parameter $\bm{\theta}=[\mbox{vec}(\bm{U});\mbox{vec}(\bm{A})]$ is $D=C(M+CL)$ according to the definition in (\ref{mshp}). 
Because the parameter $\bm{\theta}$ is nonnegative, the feasible domain $\Theta$ is a convex set. 
We assume that $\|\bm{U}\|_F^2\leq U_0$ and $\|\bm{A}\|_F^2\leq A_0$, and thus, $\|\bm{\theta}\|_2^2=\|\bm{U}\|_F^2+\|\bm{A}\|_F^2$ is bounded and $R=U_0+A_0$. 
Therefore, the condition 1 is satisfied. 

The random function $f$ is the conditional probability of an event given its history, so it is nonnegative. 
It can be rewritten as
\begin{eqnarray}
\begin{aligned}
f=\bm{X}^{\top}\bm{\theta} - \log(\bm{x}^{\top}\bm{\theta}),
\end{aligned}
\end{eqnarray}
where both $\bm{X}$ and $\bm{x}$ are $D$-dimensional nonnegative vectors, so it is convex as well. 
Furthermore, the $\bm{x}^{\top}\bm{\theta}$ corresponding to intensity function, which has a lower bound $\lambda_0$. 
Then, we have
\begin{eqnarray}\label{tau}
\begin{aligned}
\|\nabla f(\bm{\theta}) - \nabla f(\bm{\theta}')\|_2
&=\Biggl\|\frac{\bm{x}}{\bm{x}^{\top}\bm{\theta}'}-\frac{\bm{x}}{\bm{x}^{\top}\bm{\theta}}\Biggr\|_2\\
&=\Biggl\|\frac{\bm{x}\bm{x}^{\top}(\bm{\theta}-\bm{\theta}')}{(\bm{x}^{\top}\bm{\theta}')(\bm{x}^{\top}\bm{\theta})}\Biggr\|_2\\
&\leq \frac{1}{\lambda_0^2}\|\bm{x}^{\top}\bm{x}\|_2\|\bm{\theta}-\bm{\theta}'\|_2.
\end{aligned}
\end{eqnarray}
Because each $\bm{x}$ involves at most $IL$ decay kernels, whose energies have an upper bound $G$, we have $\|\bm{x}^{\top}\bm{x}\|_2\leq (ILG)^2$. 
Plugging the value into (\ref{tau}), we can find that $f$ is a $\tau$-smooth function with $\tau=(\frac{ILG}{\lambda_0})^2$. 
In other words, the function $f$ satisfies the condition 2. 

Moreover, because $f$ is a convex and $\tau$-smooth function, it is a $\tau$-strongly convex function~\citep{bertsekas2015convex}. 
Because the expected function $F=\mathbb{E}[f]$, it is also a $\tau$-strongly convex function. 
It means that the function $F$ satisfies the condition 3 with $\gamma=\tau=(\frac{ILG}{\lambda_0})^2$. 

As a result, plugging $I_{\Sigma}$, $\tau$ and $\gamma$ into (\ref{bound0}), we obtain the risk bound corresponding to the stochastic optimization of multi-agent Hawkes process:
\begin{eqnarray}\label{bound_final}
\begin{aligned}
&\mathbb{E}[f(\hat{\bm{\theta}})-f(\bm{\theta}_*)]\\
&\leq \frac{8R\tau C_{\epsilon,\delta}}{I_\Sigma}+\frac{4\sqrt{R}\sup_{f\sim \mathbb{P}}\|\nabla f(\bm{\theta}_*)\|_2\log\frac{2}{\delta}}{I_\Sigma}\\
&\quad+\frac{8 F_*\log\frac{2}{\delta}}{I_\Sigma}
+2\sqrt{R}\tau\epsilon\left(2 + \sqrt{\frac{C_{\epsilon,\delta}}{I_\Sigma}} + \frac{C_{\epsilon,\delta}}{I_\Sigma}\right).
\end{aligned}
\end{eqnarray}
Similar to \citep{zhang2017empirical}, we assume that both $F_*$ and $\sup_{f\sim \mathbb{P}}\|\nabla f(\bm{\theta}_*)\|_2$ are ignorable compared to $\tau$, and we set $\epsilon=\frac{1}{I_\Sigma}$. 
Furthermore, we assume that $\log\frac{2}{\delta}$ is comparable to $D\log I_\Sigma$. 
Then, except the first term in the right-hand side of (\ref{bound_final}), the remaining terms are ignorable. 
Plugging $C_{\epsilon,\delta}$ into (\ref{bound_final}), with high probability we have
\begin{eqnarray}
\begin{aligned}
\mathbb{E}[f(\hat{\bm{\theta}})-f(\bm{\theta}_*)]=\mathcal{O}\left( \frac{R\tau (D\log I_\Sigma+\log\frac{2}{\delta})}{I_\Sigma} \right).
\end{aligned}
\end{eqnarray}
This completes the proof.
\end{proof}

\begin{table*}[t]
\small
\centering
\caption{Summary of the Top-10 performance for various methods.}\label{tab2}
\vspace{3pt}
\setlength{\tabcolsep}{0.7pt}
\begin{tabular}
{c|
@{\hspace{2pt}}c@{\hspace{3pt}}c@{\hspace{3pt}}c@{\hspace{2pt}}|
@{\hspace{2pt}}c@{\hspace{3pt}}c@{\hspace{3pt}}c@{\hspace{2pt}}|
@{\hspace{2pt}}c@{\hspace{3pt}}c@{\hspace{3pt}}c@{\hspace{2pt}}|
@{\hspace{2pt}}c@{\hspace{3pt}}c@{\hspace{3pt}}c@{\hspace{2pt}}|
@{\hspace{2pt}}c@{\hspace{3pt}}c@{\hspace{3pt}}c@{\hspace{2pt}}|
@{\hspace{2pt}}c@{\hspace{3pt}}c@{\hspace{3pt}}c@{\hspace{1pt}}} 
\hline\hline
Method 
&\multicolumn{3}{c|@{\hspace{2pt}}}{SLIM}
&\multicolumn{3}{c|@{\hspace{2pt}}}{BPR}
&\multicolumn{3}{c|@{\hspace{2pt}}}{FPMC} 
&\multicolumn{3}{c|@{\hspace{2pt}}}{Single HP}
&\multicolumn{3}{c|@{\hspace{2pt}}}{MHPs} 
&\multicolumn{3}{c}{MHPs+Superpose}\\ \hline
Metric 
&$P$ &$R$ &$F_1$
&$P$ &$R$ &$F_1$ 
&$P$ &$R$ &$F_1$
&$P$ &$R$ &$F_1$ 
&$P$ &$R$ &$F_1$
&$P$ &$R$ &$F_1$\\ \hline
Instant Video  
&\textbf{1.47}  &\textbf{14.66} &\textbf{2.67}
&1.09  &8.90  &1.90  
&0.92  &9.22  &1.68  
&1.24  &12.39  &2.25  
&1.40  &14.00  &2.55  
&1.46  &14.64  &2.66\\ 
Android App 
&0.57  &5.69  &1.03 
&\textbf{1.15}  &\textbf{9.34}  &\textbf{2.01}  
&0.89  &8.87  &1.61  
&0.54  &5.44  &0.99  
&0.70  &7.03  &1.28  
&0.94  &9.42  &1.71\\ 
Automotive 
&0.38  &3.75  &0.68 
&\textbf{0.60}  &\textbf{5.48}  &\textbf{1.07}  
&0.48  &4.81  &0.87  
&0.22  &2.21  &0.40  
&0.22  &2.21  &0.40  
&0.44  &4.42  &0.80\\ 
Baby
&0.25  &2.51  &0.46  
&\textbf{0.38}  &\textbf{3.79}  &\textbf{0.69}  
&0.31  &3.07  &0.56  
&0.38  &3.76  &0.68  
&0.34  &3.43  &0.62  
&0.31  &3.09  &0.56\\ 
Beauty  
&0.75  &7.54  &1.37
&0.93  &9.28  &1.69
&0.40  &4.02  &0.73  
&0.87  &8.71  &1.58  
&0.89  &8.89  &1.62  
&\textbf{1.10}  &\textbf{3.10}  &\textbf{1.37} \\
Book
&0.22  &2.17  &0.39
&0.24  &2.23  &0.43
&0.20  &2.02  &0.36
&0.33  &3.25  &0.59
&0.32  &3.23  &0.58
&\textbf{0.34}  &\textbf{3.25}  &\textbf{0.60}\\ 
CDs, Vinyl 
&0.36  &3.62  &0.66 
&0.37  &3.74  &0.68  
&0.32  &3.20  &0.58  
&0.29  &2.92  &0.53  
&0.49  &4.87  &0.89
&\textbf{0.53}  &\textbf{3.91}  &\textbf{0.91} \\ 
Cell Phone  
&0.44  &4.42  &0.80
&0.80  &5.89  &1.37  
&\textbf{0.80}  &\textbf{7.96}  &\textbf{1.45}  
&0.50  &4.99  &0.91  
&0.52  &5.21  &0.95  
&0.52  &5.23  &0.95\\ 
Clothes, Jewelry  
&0.05  &0.48  &0.09
&\textbf{0.27}  &\textbf{2.16}  &\textbf{0.46}  
&0.18  &1.83  &0.33  
&0.08  &0.76  &0.14  
&0.07  &0.72  &0.13  
&0.07  &0.66  &0.12\\ 
Digital Music 
&0.82  &8.20  &1.49 
&\textbf{2.21}  &\textbf{19.33}  &\textbf{3.92}  
&1.06  &10.6  &1.93  
&0.41  &4.10  &0.75  
&0.41  &4.10  &0.75  
&1.07  &10.66  &1.94\\
Electronics
&0.15  &1.59  &0.29
&0.20  &1.81  &0.35
&0.14  &1.60  &0.29
&0.20  &1.99  &0.36
&0.22  &2.20  &0.40
&\textbf{0.23}  &\textbf{2.23}  &\textbf{0.42}\\ 
Grocery, Food  
&0.57  &5.72  &1.04
&0.70  &4.56  &1.16  
&0.45  &4.50  &0.82  
&0.62  &6.25  &1.14  
&0.62  &6.18  &1.12  
&\textbf{0.70}  &\textbf{7.03}  &\textbf{1.28}\\ 
Health Care 
&0.52  &5.24  &0.95
&\textbf{0.92}  &\textbf{4.93}  &\textbf{1.45}  
&0.22  &2.22  &0.40  
&0.67  &6.73  &1.22  
&0.67  &6.70  &1.22  
&0.67  &6.70  &1.22\\ 
Home, Kitchen 
&0.29  &2.90  &0.53 
&0.21  &1.14  &0.32  
&0.14  &1.38  &0.25  
&0.29  &2.87  &0.52  
&0.32  &3.18  &0.58  
&\textbf{0.32}  &\textbf{3.21}  &\textbf{0.60}\\ 
Kindle Store
&0.64  &6.44  &1.17  
&0.19  &0.69  &0.26  
&0.10  &1.04  &0.19  
&0.65  &6.50  &1.18  
&0.67  &6.73  &1.22  
&\textbf{0.69}  &\textbf{6.92}  &\textbf{1.26}\\ 
Movies, TV
&0.78  &7.80  &1.42  
&0.76  &5.57  &1.28  
&0.57  &5.65  &1.03  
&0.74  &7.37  &1.34
&0.80  &7.97  &1.45  
&\textbf{0.82}  &\textbf{8.17}  &\textbf{1.49}\\ 
Music Instrument 
&0.62  &6.17  &1.12 
&\textbf{1.73}  &\textbf{14.92}  &\textbf{3.05}  
&1.53  &15.34  &2.79  
&0.25  &2.47  &0.45  
&0.25  &2.47  &0.45  
&1.48  &14.81  &2.69\\ 
Office Product 
&0.46  &4.61  &0.84 
&\textbf{0.82}  &\textbf{3.32}  &\textbf{1.21}  
&0.31  &3.07  &0.56  
&0.62  &6.18  &1.12  
&0.60  &6.03  &1.10  
&0.61  &6.10  &1.11\\ 
Patio Lawn 
&0.22  &2.23  &0.41 
&0.41  &2.87  &0.70  
&0.31  &3.08  &0.56  
&0.46  &4.62  &0.84  
&0.51  &5.10  &0.93  
&\textbf{0.53}  &\textbf{5.25}  &\textbf{0.96}\\ 
Pet Supply 
&0.59  &5.89  &1.07 
&0.63  &4.11  &1.06  
&0.53  &5.26  &0.96  
&0.62  &6.18  &1.12  
&0.59  &5.94  &1.08  
&\textbf{0.64}  &\textbf{6.43}  &\textbf{1.17}\\ 
Sports 
&0.19  &1.90  &0.34   
&0.28  &2.85  &0.52  
&0.30  &2.98  &0.54  
&0.29  &2.91  &0.53  
&0.32  &3.20  &0.58
&\textbf{0.37}  &\textbf{2.56}  &\textbf{0.62}\\ 
Home Tool 
&0.15  &1.49  &0.27 
&0.48  &3.88  &0.84  
&0.44  &4.40  &0.80  
&0.16  &1.56  &0.28  
&0.16  &1.63  &0.30  
&\textbf{0.48}  &\textbf{4.82}  &\textbf{0.88}\\ 
Toys, Game 
&0.51  &5.10  &0.93 
&0.46  &3.76  &0.80  
&0.32  &3.17  &0.58  
&0.62  &6.19  &1.13  
&0.61  &6.10  &1.11  
&\textbf{0.72}  &\textbf{7.19}  &\textbf{1.31}\\ 
Video Game 
&0.53  &5.31  &0.97 
&0.76  &7.61  &1.38  
&0.73  &7.26  &1.32  
&0.53  &5.31  &0.97  
&0.62  &6.19  &1.12  
&\textbf{0.96}  &\textbf{6.88}  &\textbf{1.63}\\ \hline
Overall
&0.48 &4.81 &0.87
&0.71 &5.25 &1.19 
&0.49 &4.87 &0.89 
&0.48 &4.84 &0.88
&0.51 &5.05 &0.92 
&\textbf{0.74} &\textbf{6.43} &\textbf{1.21}\\
\hline\hline
\end{tabular}\label{tab:result}
\end{table*}

\subsection{Proof of Theorem~\ref{thm2}}
\begin{proof}
Because $N(t)=\sum_{m=1}^{M}N^m(t)$, for each entity $c\in\mathcal{C}$, its intensity is
\begin{eqnarray*}
\begin{aligned}
\lambda_c(t)=&\frac{\mathbb{E}[dN_c(t)|\mathcal{H}_t]}{dt}
=\sideset{}{_{m=1}^{M}}\sum\frac{\mathbb{E}[dN_c^m(t)|\cup_{l=1}^{M}\mathcal{H}_t^l]}{dt}\\
=&\sideset{}{_{m=1}^{M}}\sum\frac{\mathbb{E}[dN_c^m(t)|\mathcal{H}_t^m]}{dt}
=\sideset{}{_{m=1}^{M}}\sum\lambda_c^m(t).
\end{aligned}
\end{eqnarray*}
Here $\mathcal{H}_t=\cup_{m=1}^{M}\mathcal{H}_t^{m}$ contains all historical events in the superposed process. 
Because the Hawkes processes have shared impact functions, we have
\begin{eqnarray*}
\begin{aligned}
\lambda_c(t) =& \sideset{}{_{m=1}^{M}}\sum\Bigl(\mu_c^m(t) +\sideset{}{_{(t_i^m, c_i^m)\in\mathcal{H}_t^{m}}}\sum\phi_{cc_i^m}(t-t_i^m)\Bigr)\\
=&\sideset{}{_{m=1}^{M}}\sum\mu_c^m(t) + \sideset{}{_{(t_i, c_i)\in\mathcal{H}_t}}\sum\phi_{cc_i}(t-t_i),
\end{aligned}
\end{eqnarray*}
for $c\in\mathcal{C}$. 
According to the definition in (\ref{mshp}), we have $N(t)\sim HP(\sum_{m=1}^{M}\bm{\mu}^m,\bm{A})$.
\end{proof}

\subsection{More Experimental Results}
For each method in our paper, we generate a recommendation list $\bm{r}^m$ with the top-$N$ most possible items for each user. 
Suppose the real set of the items that user $m$ will buy is $\bm{t}^m$. 
The top-$N$ precision ($P@N$), recall ($R@N$) and $F_{1}$-score ($F_1@N$) are used for evaluation.
Their definitions are.
\begin{eqnarray*}
\begin{aligned}
P@N  &= \frac{1}{M}\sum_{m} P_m@N =\frac{1}{M}\sum_{m} \frac{|\bm{r}^m \cap \bm{t}^m|}{|\bm{r}^m|}\times 100\%\\
R@N &= \frac{1}{M}\sum_{m} R_m@N =\frac{1}{M}\sum_{m} \frac{|\bm{r}^m \cap \bm{t}^m|}{|\bm{t}^m|}\times 100\%\\
F_{1}@N &= \frac{1}{M}\sum_{m} F_{1m}@N  =\frac{1}{M}\sum_{m} \frac{2\cdot P_m@N \cdot R_m@N}{P_m@N+R_m@N}
\end{aligned}
\end{eqnarray*}

Besides the Top-5 performance of various methods, their Top-10 ($N=10$) performance is listed in Table~\ref{tab2}. 
We can find that our ``MHPs + Superpose'' method is still superior to other methods in most situations.

\subsection{Relation to Item-To-Item Recommendation}
If we specialize the impact function $\phi_{cc_i}(t,t_i)$ as some similarity measurement (\emph{e.g.}, cosine), our models are reduced to item-to-item (I2I) recommendation methods.
However, different from existing I2I models, such as~\citep{christakopoulou2016local,zheng2014cslim,ning2011slim}, we also include user personalized preference in the modeling process by $\mu_c^m$, aiming to model the intuition of ``different people may have different preferences even conditioned on the same product'' 
Further more, we considered the time influence in the item-to-item transition process, which is more reasonable for real-world scenarios, for example, we can use exponential decay function to model the changes of relevance over time, and also we can design more complex impact function to capture more detail item-to-item transition patterns.
Overall, our models are very flexible, and can degenerate into many other popular recommendation models, which also inspires us to design other specific algorithms for capturing user behavior patterns in different real-world scenarios.



\end{document}